\documentclass[10pt,journal,compsoc]{IEEEtran}

\usepackage{tikz}
\usepackage{pgfplots}
\usetikzlibrary{calc}
\usepackage{fancyhdr}
\usepackage{times}
\usepackage{epsfig}
\usepackage{graphicx}
\usepackage{amsmath}
\usepackage{amssymb}
\usepackage{multirow}
\usepackage{verbatim}
\usepackage{color}
\usepackage{aliascnt}
\usepackage{amsthm}
\usepackage{caption}
\usepackage{subcaption}
\usepackage{multirow}
\usepackage{wrapfig}
\usepackage{footnote}
\usepackage{cancel}
\usepackage{array}
\usepackage[normalem]{ulem}

\usepackage[ruled,titlenotnumbered]{algorithm2e}

 \makeatletter
\def\ifundefined{\@ifundefined}
\makeatother

\newtheorem{example}{\sc Example}
\newtheorem{lem}{Lemma}

\newtheorem{theor}{Theorem}

\def\utc{\mbox{$\;\overset{c}{=}\;$}}
\def\approxutc{\mbox{$\;\overset{c}{\approx}\;$}}
\def\equivutc{\mbox{$\;\overset{c}{\equiv}\;$}}

\def\P{\mbox{${\cal P}$}}

\def\Pk{\mbox{${\cal P}_{\Sigma}$}}

\def\eg{{\em e.g.~}}
\def\ie{{\em i.e.~}}

\def\Real{{\mathbb R}}

\def\x{{\mathbf x}}

\def\Real{\mathcal{R}}

\def\df{\,\mathrm{d}}

\def\af{s}

\def\dS{\mbox{$\rho^{\af}_1$}}
\def\dSbar{\mbox{$\rho^{\af}_2$}}
\def\dSk{\mbox{$\rho^{\af}_k$}}

\def\rbf{\mbox{$\psi$}}
\def\rsmall/{\mbox{$r${\it{-small}}}}

\def\thebias/{Breiman's bias}

\usepackage{MnSymbol}
\def\const{\mbox{$\mathit{const}$}}            

\newcommand{\E}{{\mathbf E}}

\def\KNN/{\mbox{$K\!N\!N$}}
\def\KM/{{\sf KM}}
\def\pKM/{{\color{blue}{\sf pKM}}}
\def\kKM/{{{\color{blue}$k${\sf KM}}}}
\def\wkKM/{{\color{blue}$wk${\sf KM}}}
\def\AA/{{\sf AA}}
\def\AD/{{\sf AD}}
\def\NC/{{\sf NC}}
\def\AC/{{\sf AC}}

\ifx\undefined\NEW
	\def\NEW{ }
\fi

\usepackage[pagebackref=true,breaklinks=true,letterpaper=true,colorlinks,bookmarks=false]{hyperref}

\begin{document}

\title{Kernel clustering: density biases and solutions}

\author{Dmitrii Marin$^*$ \hspace{5ex} Meng Tang$^*$ \hspace{5ex} 
Ismail Ben Ayed$^{\dag}$ \hspace{5ex}  Yuri Boykov$^*$ \\[0.5ex]
{\small $^*$Computer Science, University of Western Ontario, Canada  \hspace{3ex} $^\dag$\'Ecole de Technologie Sup\'erieure, University of Quebec, Canada} \\[-1ex]
{\tt\small dmitrii.a.marin@gmail.com     mtang73@csd.uwo.ca    ismail.benayed@etsmtl.ca   yuri@csd.uwo.ca}
 \vskip -1cm}

\maketitle
 \vskip -1cm
\begin{abstract}
Kernel methods are popular in clustering due to their generality and discriminating power.  
However, we show that many kernel clustering criteria have {\em density biases}
theoretically explaining some practically significant artifacts 
empirically observed in the past. 
For example, we provide conditions and formally prove the {\em density mode isolation} 
bias in kernel K-means for a common class of kernels.
We call it \thebias/ due to its similarity to the {\em histogram mode isolation} 
previously discovered by Breiman in decision tree learning with Gini impurity.
We also extend our analysis to other popular kernel clustering methods, 
\eg average/normalized cut or dominant sets, where density biases 
can take different forms. For example, splitting isolated points by cut-based criteria 
is essentially the sparsest subset bias, which is the opposite of the density mode bias.
Our findings suggest that a principled solution for density biases in kernel clustering
should directly address data inhomogeneity. 
We show that {\em density equalization} can be implicitly achieved using either 
locally adaptive weights or locally adaptive kernels. 
Moreover, density equalization makes many popular kernel clustering objectives equivalent.
Our synthetic and real data experiments illustrate density biases and proposed solutions.
We anticipate that theoretical understanding of kernel clustering limitations and their principled solutions
will be important for a broad spectrum of data analysis applications across the disciplines.

\end{abstract}

\section{Introduction}

\begin{figure}[t]
        \centering
        \begin{tabular}{@{\extracolsep{0.2ex}}l@{\extracolsep{1ex}}cc}
        \multirow{2}{*}[14ex]{   \rotatebox[origin=c]{90}{ \it uniform density data}   }    &
        \includegraphics[width=1.3in]{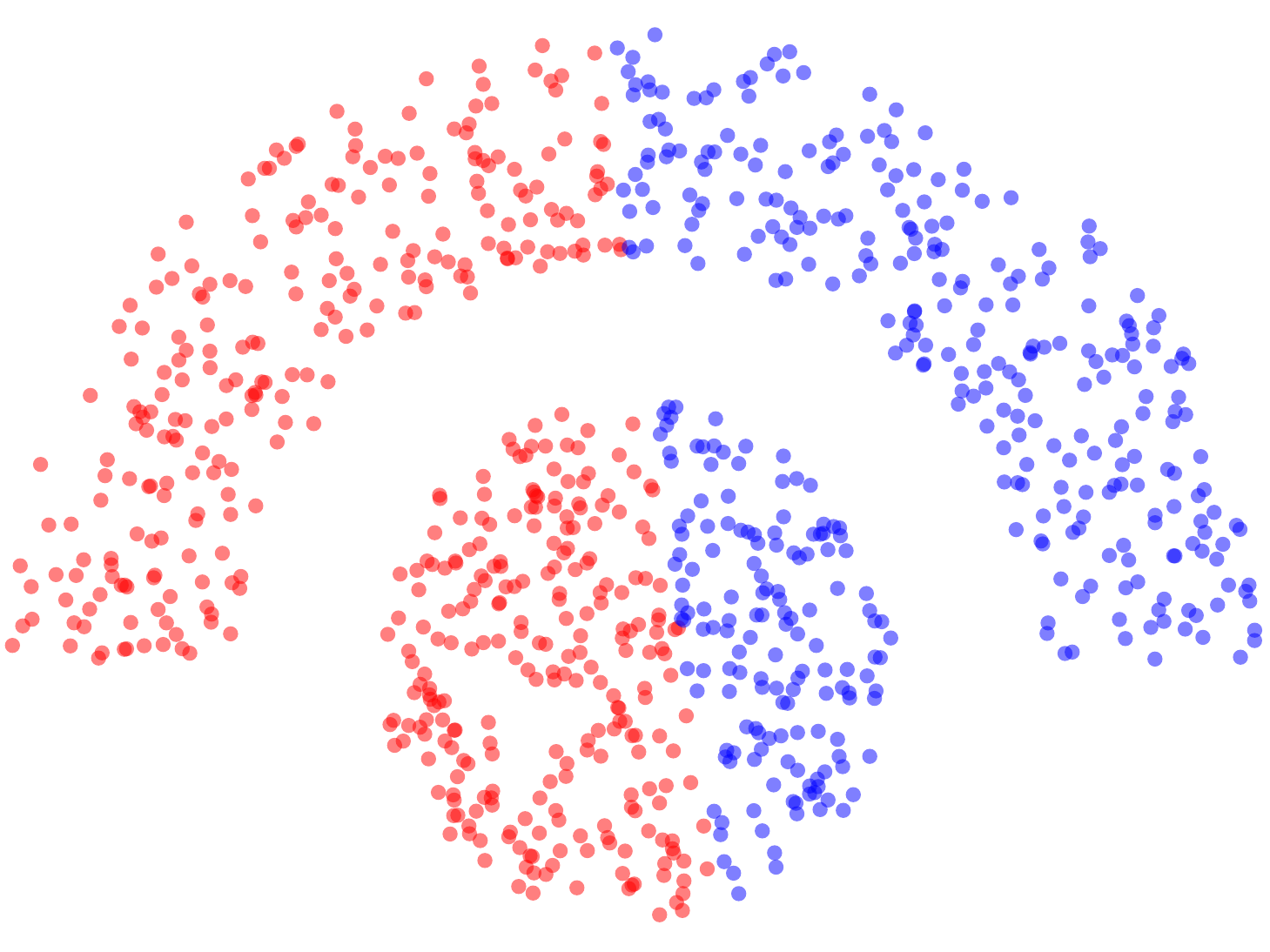} &
        \includegraphics[width=1.3in]{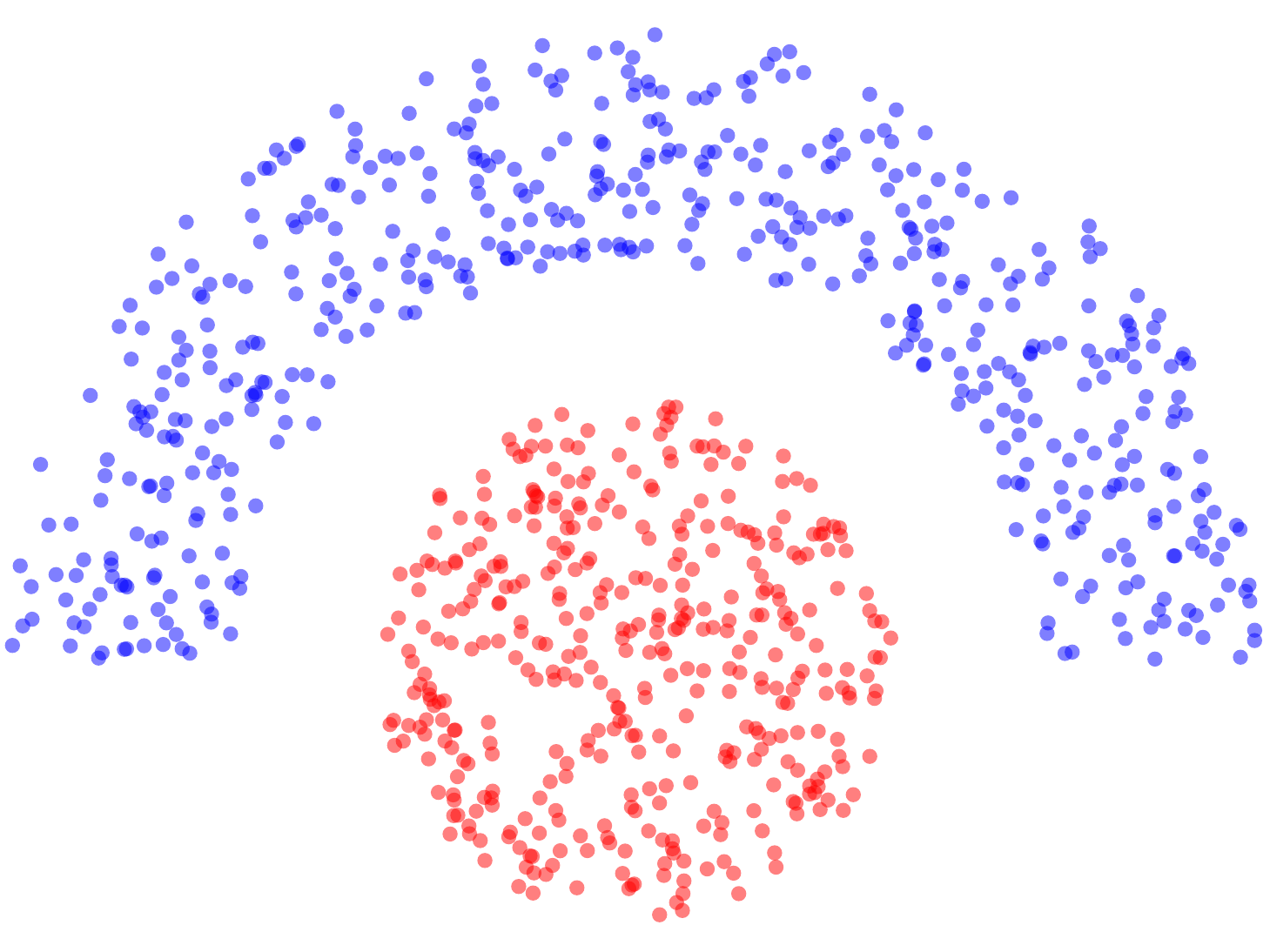} \\[-1ex]
        & (a) K-means & (b) kernel K-means \\[2ex]
       \multirow{1}{*}[13ex]{  \rotatebox[origin=c]{90}{ \it non-uniform data} } 
       &\includegraphics[width=1.5in]{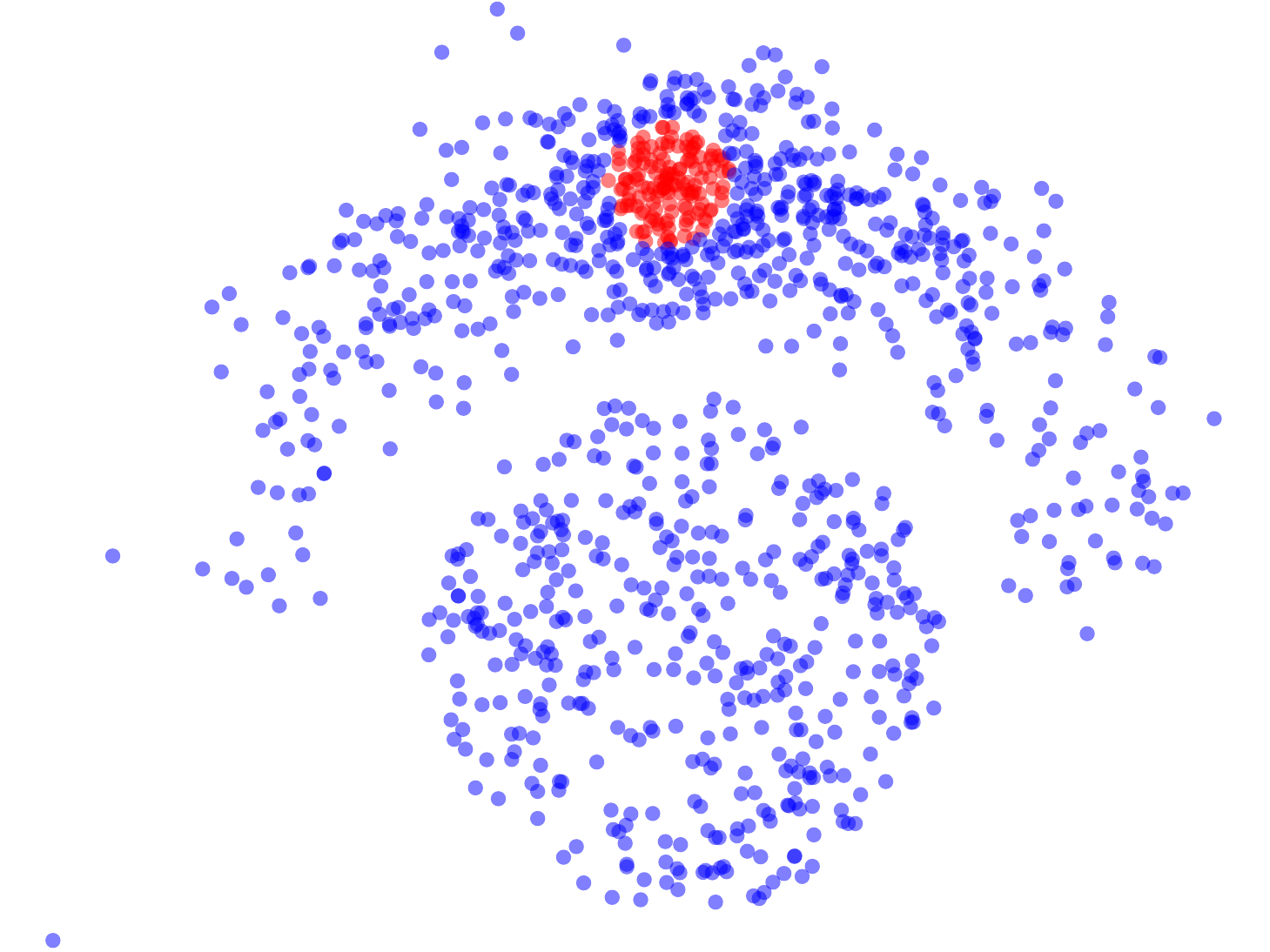}&
        \includegraphics[width=1.5in]{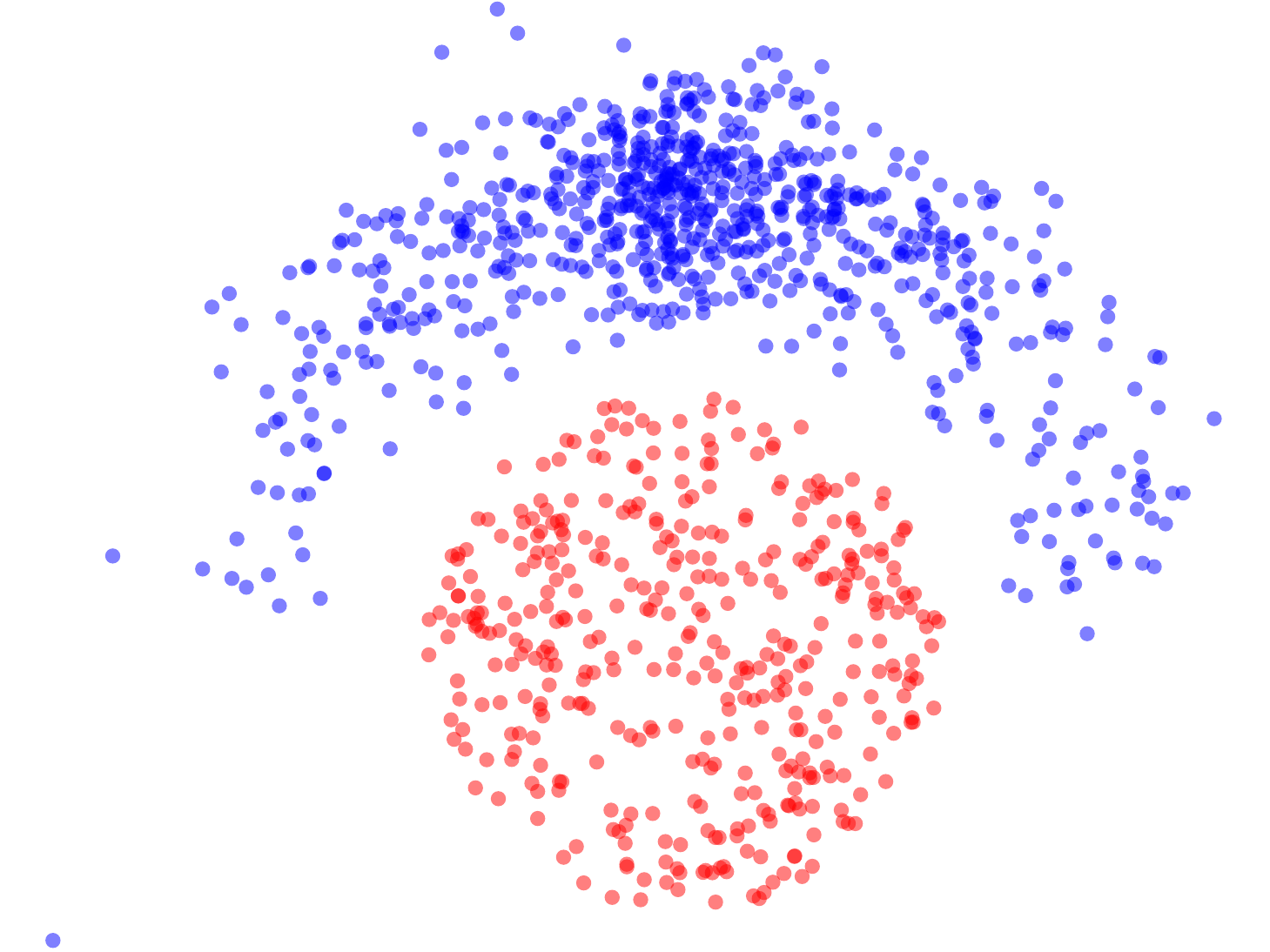} \\[-0.5ex]
        & (c) kernel K-means
        & (d) kernel clustering\\
         & \footnotesize (\thebias/, mode isolation) &  \footnotesize (adaptive weights or kernels) 
         \end{tabular}
        \caption{\label{fig:synexp} Kernel K-means with Gaussian kernel \eqref{eq:gauss_k} gives desirable nonlinear separation for {\em uniform} density clusters (a,b). But, 
        for {\em non-uniform} clusters in (c) it either isolates a small dense ``clump'' for smaller $\sigma$ due to \thebias/ (Section~\ref{sec:analysis}) or gives
        results like (a) for larger $\sigma$. No fixed $\sigma$ yields solution (d) given by locally adaptive kernels or weights eliminating the bias (Sections~\ref{sec:Bandwidth selection} \& \ref{sec:adaptive weighting}).%
        }
\end{figure}

In machine learning, \emph{kernel clustering} is a well established data analysis technique \cite{scholkopf1998kernelpca,vapnik:98,Shi2000,Ksurvey:TNN01,zhang2002large,Girolami2002,dhillon2004kernel,Marcello2007,Chitta2011,Hartley:pami2015kernels}
that can identify non-linearly separable structures, see Figure~\ref{fig:synexp}(a-b). Section~\ref{sec:kkm overview} reviews the kernel K-means 
and related clustering objectives, some of which have theoretically explained biases, see Section~\ref{sec:other criteria}.
In particular, Section~\ref{sec:gini overview} describes the discrete {\em Gini clustering criterion} standard  in decision tree learning
where Breiman \cite{breiman1996} proved a bias to histogram mode isolation. 

{

Empirically, it is well known that kernel K-means or {\em average association} (see Section~\ref{sec:Graph_Clustering})
has a bias to so-called ``tight'' clusters for small bandwidths~\cite{Shi2000}. Figure~\ref{fig:synexp}(c) demonstrates this
bias on a non-uniform modification of a typical toy example for kernel K-means with common Gaussian kernel
\begin{equation} \label{eq:gauss_k} 
k(x,y)\propto \exp \left(  -\frac{\|x-y\|^2}{2\sigma^2} \right).
\end{equation}
This paper shows in Section~\ref{sec:analysis}  that under certain conditions kernel K-means approximates the {\em continuous} 
generalization of the Gini criterion where we formally prove a mode isolation bias  similar to the discrete case analyzed by Breiman. 
Thus, we refer to the ``tight'' clusters in kernel K-means as {\em \thebias/}. 
 
We propose a {\em density equalization} principle directly addressing the cause of \thebias/. 
First, Section \ref{sec:adaptive weighting} discusses modification of the density with adaptive point weights. 
Then, Section~\ref{sec:Bandwidth selection} {\NEW shows that a general class of locally adaptive 
{\em geodesic kernels}~\cite{Hartley:pami2015kernels} implicitly transforms data and modifies its density. }
We derive ``density laws'' relating adaptive weights and kernels to density transformations. 
They allow to implement {\em density equalization} resolving \thebias/, see Figure~\ref{fig:synexp}(d). 
One popular heuristic~\cite{zelnik2004self} approximates a special case of our Riemannian kernels. 

Besides mode isolation, kernel clustering may have the opposite density bias,
\eg {\em sparse subsets} in Normalized Cut \cite{Shi2000}, see Figure~\ref{fig:nc fails}(a).
Section~\ref{sec:NC and BB} presents ``normalization'' as implicit \emph{density inversion} 
establishing a formal relation between sparse subsets and \thebias/. 
Equalization addresses any density biases. Interestingly,
density equalization makes many standard kernel clustering criteria conceptually equivalent, see Section~\ref{sec:discussion}.

}

\begin{figure*}[t]
\begin{tabular}{cc}
\includegraphics[width=0.55\columnwidth]{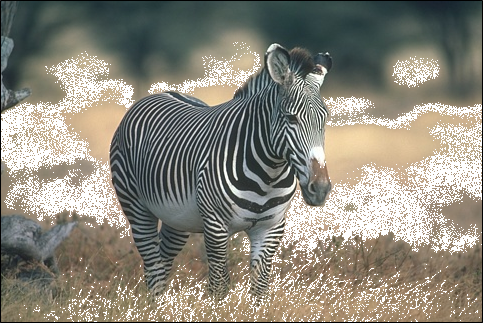}  
\includegraphics[width=0.45\columnwidth]{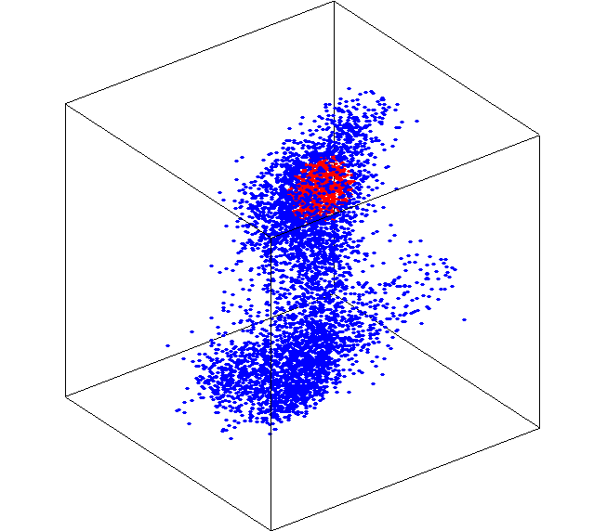}  &
\includegraphics[width=0.55\columnwidth]{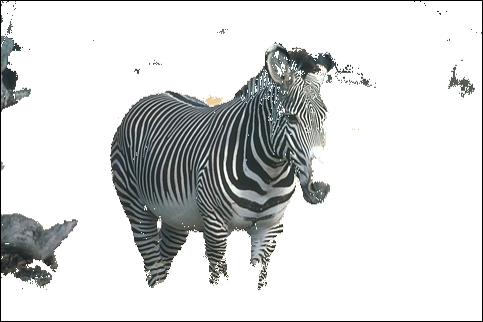}  
\includegraphics[width=0.45\columnwidth]{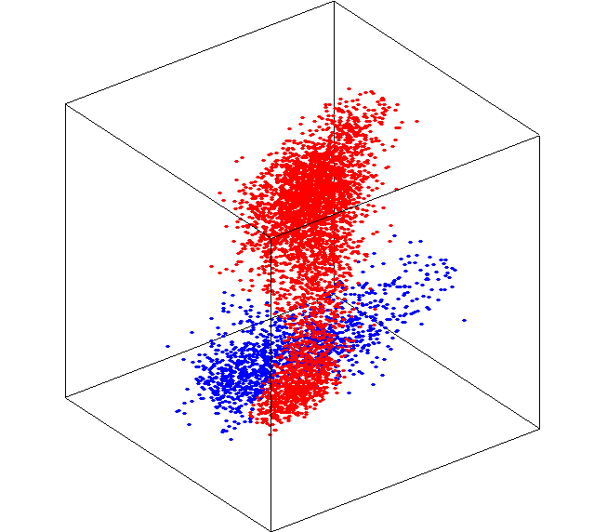} \\
(a) \thebias/ & (b) good clustering 
\end{tabular}
\caption{Example of \thebias/ on real data. Feature vectors are 3-dimensional LAB colours corresponding to image pixels. 
Clustering results are shown in two ways. First, {\em red} and {\em blue} show different clusters inside LAB space. 
Second, pixels with colours in the ``background'' (red) cluster are removed from the original image.
(a) shows the result for kernel K-means with a fixed-width Gaussian kernel isolating a small dense group of pixels from the rest. 
(b) shows the result for an adaptive kernel, see Section \ref{sec:Bandwidth selection}.}
\label{fig:bias example}
\end{figure*}

\subsection{Kernel K-means} \label{sec:kkm overview}

A popular data clustering technique, {\em kernel K-means} \cite{scholkopf1998kernelpca} 
is a generalization of the basic {\em K-means} method. 
{   Assuming $\Omega$ denotes a finite set of points and $f_p \in \Real^N$ is a feature (vector) for point~$p$, 
the basic K-means minimizes the sum of squared errors within clusters, 
that is, distances from points $f_p$ in each cluster $S_k\subset\Omega$ to the cluster means $m_k$
\begin{equation} \label{eq:kmeans}
\left(\parbox{9ex}{\centering\bf\small k-means criterion}\right)\quad\quad\quad\quad\quad\quad
\sum_k \sum_{p \in S^k} \|f_p - m_k \|^2.\quad\quad\quad\quad
\end{equation}} 

Instead of clustering data points $\{f_p \; | \; p\in\Omega\}\subset \mathcal{R}^N$
in their original space, kernel K-means uses mapping 
$\phi \;\; : \;\; \mathcal{R}^N\rightarrow\mathcal{H}$ embedding 
input data $f_p \in \mathcal{R}^N$ as points $\phi_p \equiv \phi(f_p)$ in a higher-dimensional Hilbert space $\mathcal{H}$. 
Kernel K-means minimizes the sum of squared errors in the embedding space corresponding to
the following (mixed) objective function 
\begin{equation} \label{eq:kmixed}
F(S, m) \;\;=\;\; \sum_k \sum_{p \in S^k} \|\phi_p - m_k \|^2
\end{equation}
where $S=(S^1,S^2,\dots,S^K)$ is a partitioning (clustering) of $\Omega$ into $K$ clusters, $m=(m_1, m_2, \dots m_K)$ 
is a set of parameters for the clusters, and $\|.\|$ denotes the {  Hilbertian} norm\footnote{  Our later examples use finite-dimensional embeddings $\phi$ where 
$\mathcal{H}=\mathcal{R}^M$ is an Euclidean space ($M\gg N$) and $\|.\|$ is the Euclidean norm.}. 
Kernel K-means finds clusters separated by hyperplanes in $\mathcal{H}$. 
In general, these hyperplanes correspond to non-linear surfaces in the original input space $\mathcal{R}^N$. 
In contrast to~\eqref{eq:kmixed}, standard K-means objective \eqref{eq:kmeans}
is able to identify only linearly separable clusters in $\mathcal{R}^N$. 

Optimizing $F$ with respect to the parameters yields closed-form solutions corresponding to the cluster means 
in the embedding space:
\begin{equation} \label{eq:mu_phi}
\hat{m}_k = \frac{\sum_{q \in S^k} \phi_q}{|S^k|}
\end{equation}
where $|.|$ denotes the cardinality (number of points) in a cluster.  
Plugging optimal means \eqref{eq:mu_phi} into objective \eqref{eq:kmixed} yields a high-order function, which depends solely 
on the partition variable $S$: 
\begin{equation}
\label{eq:kKmeans1}
 F(S) \;\;=\;\; \sum_k \sum_{p \in S^k} \left \| \phi_p-  \frac{\sum_{q \in S^k} \phi_q}{|S^k|} \right \|^2.
\end{equation} 
Expanding the Euclidean distances in \eqref{eq:kKmeans1}, one can obtain an equivalent 
pairwise clustering criterion expressed solely in terms of inner products $\langle \phi(f_p), \phi(f_q) \rangle$ 
in the embedding space $\mathcal{H}$:
\begin{equation} 
F(S) \;\;\utc\;\;  - \sum_k \frac{\sum_{pq \in S^k}\langle \phi(f_p), \phi(f_q) \rangle}{|S^k|} 
\end{equation} 
where $\utc$ means equality up to an additive constant. The inner product is often replaced with kernel $k$, a symmetric function:
\begin{equation} 
\label{eq:dotproduct} 
k(x,y) := \langle\phi(x),\phi(y)\rangle.
\end{equation}
Then, kernel K-means objective \eqref{eq:kKmeans1} can be presented as
\begin{equation} \label{eq:kKmeans3}
\left(\parbox{9ex}{\centering\bf\small kernel \\ k-means criterion}\right)\quad\quad\quad
F(S) \;\;\utc\;\;  - \sum_k \frac{\sum_{pq \in S^k}k(f_p,f_q)}{|S^k|}. \quad
\end{equation}

Formulation~\eqref{eq:kKmeans3} enables optimization in high-dimensional space $\mathcal H$ that only uses kernel computation and does not require computing the embedding $\phi(x)$. Given a kernel function, one can use the kernel K-means without knowing the corresponding embedding. However, not any symmetric function corresponds to the inner product in some space. Mercer's theorem \cite{vapnik:98} states that any \emph{positive semidefinite} (p.s.d.) kernel function $k(x,y)$ can be expressed as an inner product in a higher-dimensional space.
While p.s.d. is a common assumption for kernels, pairwise clustering objective \eqref{eq:kKmeans3} is often extended beyond
p.s.d. affinities. There are many other extension of kernel K-means criterion \eqref{eq:kKmeans3}.
Despite the connection to density modes made in our paper, kernel clustering has only a weak relation to {\em mean-shift} \cite{meanshift:02}, 
\eg see \cite{KC:arXiv16}.

\subsubsection{Related graph clustering criteria} \label{sec:Graph_Clustering}

Positive semidefinite kernel $k(f_p, f_q)$ in \eqref{eq:kKmeans3}  can be replaced by an arbitrary 
pairwise similarity or affinity matrix $A=[A_{pq}]$. This yields the {\em average association} criterion, which 
is known in the context of graph clustering \cite{Shi2000,buhmann:pami03,dhillon2004kernel}:
\begin{equation} \label{eq:AA}
 -\sum_k \frac{\sum_{pq \in S^k} A_{pq}}{|S^k|}.
\end{equation}

The standard kernel K-means algorithm \cite{dhillon2004kernel,Chitta2011} is not guaranteed to decrease \eqref{eq:AA} 
for improper (non p.s.d.) kernel $k := A$. However,~\cite{buhmann:pami03} showed that dropping  p.s.d. assumption
is not essential: for arbitrary association $A$ {  there is a p.s.d. kernel $k$ 
such that objective \eqref{eq:kKmeans3} is} equivalent  to \eqref{eq:AA} up to a constant.

In~\cite{Shi2000} authors experimentally observed that the average association~\eqref{eq:AA} or 
kernel K-means \eqref{eq:kKmeans3} objectives have a bias to separate small dense group of data points 
from the rest, \eg see Figure~\ref{fig:bias example}. 

Besides average association, there are other pairwise graph clustering criteria related to kernel K-means. 
{\em Normalized cut} is a common objective in the context of spectral clustering \cite{Shi2000,von:tutorial2007}. It optimizes the following objective 
\begin{equation}\label{eq:NC}
 -\sum_k \frac{\sum_{pq \in S^k} A_{pq}}{\sum_{p \in S^k}d_p} .
\end{equation}
where $ d_p = \sum_{q \in \Omega} A_{pq} $. Note that for $d_p=1$ equation~\eqref{eq:NC} reduces to~\eqref{eq:AA}.
It is known that Normalized cut objective is equivalent to a weighted version of kernel K-means criterion
\cite{bach:nips03,dhillon2004kernel}.

\subsubsection{Probabilistic interpretation via kernel densities}   \label{sec:prob inter}

Besides {\em kernel clustering}, kernels are also commonly used for {\em probability density estimation}. 
This section relates these two independent problems.
Standard \emph{multivariate kernel density estimate} or {\em Parzen density estimate} for the distribution of data points 
within cluster $S^k$ can be  expressed as follows \cite{Bishop06}:  
\begin{equation} \label{eq:Parzen_estimate}
\Pk (x|S^k) \;\; := \;\; \frac{\sum_{q \in S^k}k(x,f_q)}{|S^k|}, 
\end{equation}  
with kernel $k$ having the form:
\begin{equation} \label{eq:density kernel}
k(x,y) \;\;=\;\; |\Sigma|^{-\frac12}\; \rbf \left (\Sigma^{-\frac12}(x-y) \right)
\end{equation}
where $\rbf$ is a symmetric multivariate density and $\Sigma$ is a symmetric positive definite {\em bandwidth} matrix 
controlling the density estimator's smoothness. 
One standard example is the Gaussian (normal) kernel \eqref{eq:gauss_k} corresponding to 
\begin{equation} \label{eq:normal kernel}
       \rbf(t) \;\propto\;  \exp\left (- \frac{\|t\|^2}{2\;} \right ),
\end{equation} 
which is commonly used both in kernel density estimation~\cite{Bishop06} and kernel clustering~\cite{Girolami2002,Shi2000}. 

The choice of bandwidth $\Sigma$ is crucial for accurate density estimation, while the choice of $\psi$ plays only 
a minor role~\cite{scott1992multivariate}. There are numerous works regarding kernel selection for accurate density estimation
using either fixed~\cite{silverman1986density,scott1992multivariate,izenman1991review}
{\NEW or variable bandwidth~\cite{TarrelScott92}.}
 For example, Scott's \emph{rule of thumb} is 
\begin{equation}
\sqrt{\Sigma_{ii}} = \frac{r_i}{\sqrt[N+4]{n}}, \quad\quad \Sigma_{ij}=0  \text{ for } i\ne j 
\end{equation}
where $n$ is the number of points, and $r_i^2$ is the variance of the $i$-th feature that could be interpreted as the range or scale of the data. 
Scott's rule gives optimal {\em mean integrated squared error} for normal data distribution, but in practice it works well in more general settings.
In all cases the optimal bandwidth for sufficiently large datasets 
is a small fraction of the data range~\cite{duda1973pattern2,Bishop06}. 
For shortness, we use adjective \rsmall/ to describe bandwidths providing accurate density estimation.

If kernel $k$ has form~\eqref{eq:density kernel} up to a positive multiplicative constant then kernel K-means objective \eqref{eq:kKmeans3} 
can be expressed in terms of kernel densities \eqref{eq:Parzen_estimate} for points in each cluster~\cite{Girolami2002}:
\begin{equation} \label{eq:Parzen_energy}
F(S) \;\utc\; -\sum_k \sum_{p \in S^k} \Pk (f_p|S^k).
\end{equation}

\subsection{Other clustering criteria and their known biases}\label{sec:other criteria}

{ One of the goals of this paper is a theoretical explanation for the bias of kernel K-means with small bandwidths toward tight dense clusters, which we call {\em \thebias/}, see Figs~\ref{fig:synexp}-\ref{fig:bias example}. This bias was observed in the past only empirically. As discussed in Section~\ref{sec:extreme}, large bandwidth reduces kernel K-means to basic K-means where bias to equal  cardinality clusters is known \cite{UAI:97}.
This section reviews other standard clustering objectives, entropy and Gini criteria, 
that have biases already well-understood theoretically. In Section~\ref{sec:analysis} we establish a connection between
Gini clustering and kernel K-means in case of \rsmall/  kernels. 
This connection allows theoretical analysis of \thebias/ in kernel K-means.
}

\subsubsection{Probabilistic K-means and entropy criterion}

Besides non-parametric kernel K-means clustering there are well-known parametric extensions of basic K-means \eqref{eq:kmeans} 
based on probability models. {\em Probabilistic K-means} \cite{UAI:97} or
{\em model based clustering}~\cite{Fraley2002} use some given likelihood functions $P ( f_p | \theta_k )$ 
instead of distances $\|f_p - \theta_k\|^2$ in \eqref{eq:kmeans} as in clustering objective
\begin{equation} 
\label{eq:pkm_energy}
 -\sum_k \sum_{p \in S^k} \log P ( f_p | \theta_k ).
\end{equation}
Note that objective~\eqref{eq:pkm_energy} reduces to basic K-means \eqref{eq:kmeans} 
for Gaussian probability model $P(.|\theta_k)$ with mean $\theta_k$ and a fixed scalar covariance matrix.

In probabilistic K-means  \eqref{eq:pkm_energy} models can differ from Gaussians depending on {\em a priori} assumptions
about the data in each cluster, \eg gamma, Gibbs, or other distributions can be used. For more complex data, each cluster 
can be described by highly-descriptive parametric models such as Gaussian mixtures (GMM). Instead of kernel density estimates in kernel K-means~\eqref{eq:Parzen_energy}, probabilistic K-means \eqref{eq:pkm_energy} uses parametric distribution models. 
Another difference is the absence of the $\log$ in \eqref{eq:Parzen_energy} compared to \eqref{eq:pkm_energy}. 

The analysis in~\cite{UAI:97} shows that in case of highly descriptive model~$P$, \eg GMM or histograms,~\eqref{eq:pkm_energy} can be approximated by the standard \emph{entropy criterion} for clustering:
\begin{equation}
\label{entropy_criterion}
\left(\parbox{9ex}{\centering\bf\small entropy \\  criterion}\right)\quad\quad\quad \quad\quad \sum_k | S^k | \cdot H ( S^k ) \quad\quad\quad\quad\quad\quad\quad
\end{equation}
where $H ( S^k )$ is the entropy of the distribution of the data in~$S^k$:
$$  H ( S^k ) \;\;:=\;\; -\int  P (x | \theta_k ) \log  P ( x | \theta_k ) \df x. $$
The discrete version of the entropy criterion is widely used for learning binary decision trees 
in classification \cite{breiman1996,Bishop06,Criminisi2013}. {  It is known that the entropy criterion above 
is biased toward equal size clusters \cite{breiman1996,UAI:97,volbias:ICCV15}.}

\subsubsection{Discrete Gini impurity and criterion}\label{sec:gini overview}

Both Gini and entropy clustering criteria are widely used in the context of decision trees \cite{Bishop06,Criminisi2013}. 
These criteria are used to decide the best split at a given node of a binary classification tree \cite{breiman1984classification}. 
The Gini criterion can be written for clustering $\{S^k\}$ as
\begin{align} 
\left(\parbox{13ex}{\centering\bf\small discrete \\ Gini criterion}\right) \quad\quad\quad\quad\quad \sum_k|S^k|\cdot G(S^k) \quad\quad\quad\quad\quad \label{eq:gini_criterion}
\end{align}
where $G(S^k)$ is the {\em Gini impurity} for the points in $S^k$. Assuming discrete feature space $\cal L$ instead of $\Real^N$, the Gini impurity is
\begin{equation} \label{eq:dGini_impurity}
G(S^k)\;:=\;1-\sum_{l\in\cal L} \P (l\,|S^k)^2  
\end{equation}
where $\P(\cdot\,|S^k)$ is the empirical probability (histogram) of discrete-valued features $f_p\in \cal L$ in cluster $S^k$. 

Similarly to the entropy, Gini impurity $G(S^k)$  can be viewed as a measure of sparsity or ``peakedness'' of the distribution for points in $S^k$. Note that \eqref{eq:gini_criterion} has a form similar to the entropy criterion in \eqref{entropy_criterion}, except that entropy $H$ is replaced by the Gini impurity. Breiman~\cite{breiman1996} analyzed the theoretical properties of the discrete Gini criterion \eqref{eq:gini_criterion} when $\P(\cdot\,|S^k)$ are {\em discrete histograms}. He proved 
\begin{tabular}{cp{0.59\columnwidth}} \\[-5ex]
\raisebox{-1.15\height}{\includegraphics[width=0.35\columnwidth,clip,trim=19mm 10cm 19mm 15mm]{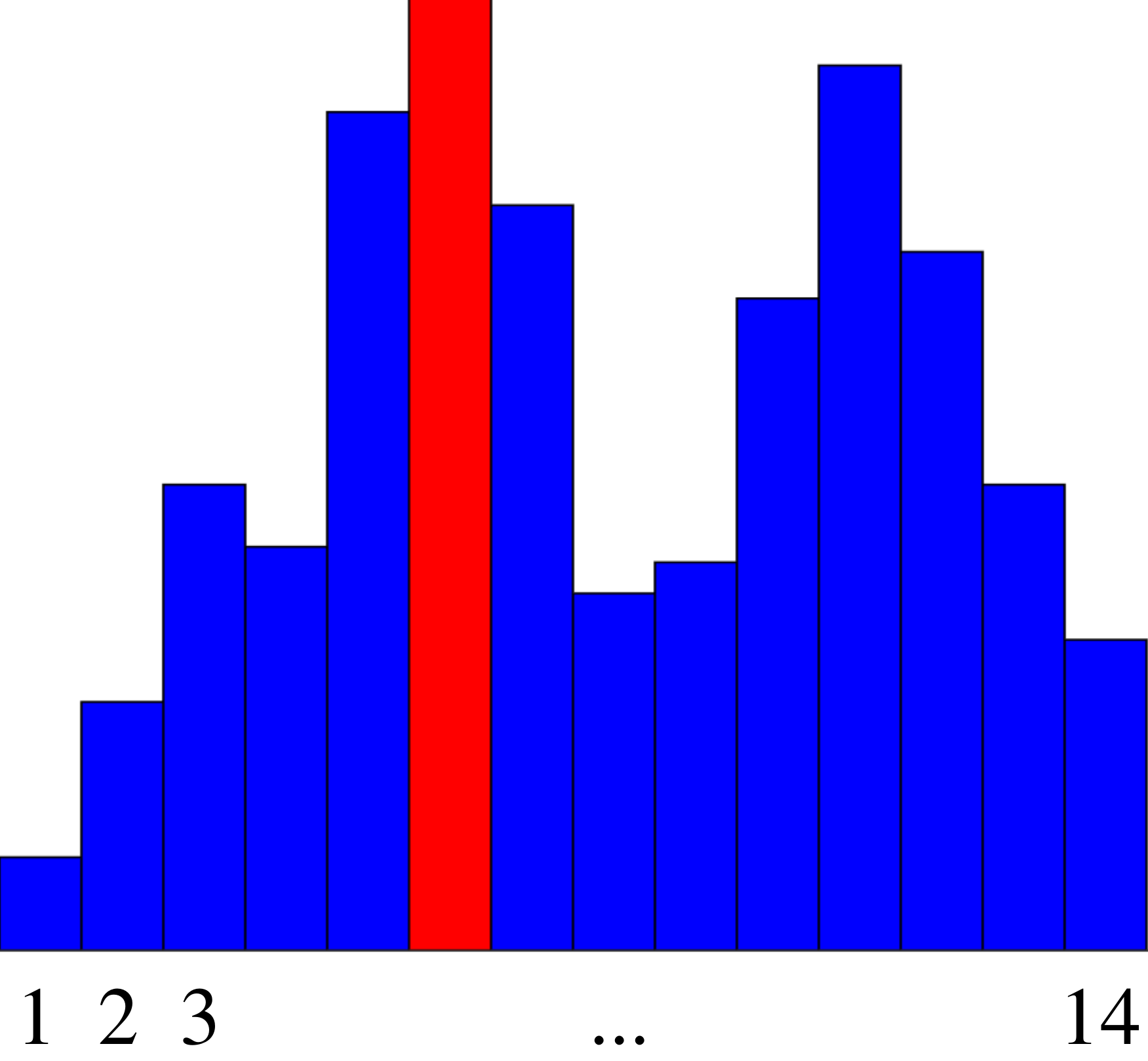}}
&
\begin{theor}[Breiman] \label{th:breiman}
For $K=2$ the minimum of the Gini criterion \eqref{eq:gini_criterion}
for discrete Gini impurity \eqref{eq:dGini_impurity} is achieved by assigning all data points with 
the highest-probability feature value in $\mathcal L$ to one cluster and the remaining data points 
to the other cluster, as in example for $\mathcal L=\{1,\dots,14\}$ on the left. \qed
\end{theor}
\end{tabular}


\section{Breiman's bias (numerical features)} \label{sec:analysis}

In this section we show that the kernel K-means objective reduces to a {\NEW novel} \emph{continuous} Gini criterion under some general conditions on the kernel function, {\NEW see Section~\ref{sec:kKM=Gini}}. We formally prove {\NEW in Section~\ref{app:gini}} that the optimum of the continuous Gini criterion isolates the data density mode. That is, we show that the discussed earlier biases observed in the context of clustering~\cite{Shi2000} and decision tree learning~\cite{breiman1996} are the same phenomena. 
{\NEW Section~\ref{sec:dsets} establishes connection to  maximum cliques~\cite{motzkin1965maxima} and {\em dominant sets}~\cite{Marcello2007}}.

{  For further analysis we reformulate the problem of clustering a discrete set of points $\{f_p\,|\,p\in\Omega\}\subset\Real^N$, see Section~\ref{sec:kkm overview},
as a continuous domain clustering problem. 
{\NEW Let $P$ be a probability measure over domain $\Real^N$ and $\rho$ be the corresponding continuous probability density function} such that the discrete points $f_p$ could be treated as samples from this distribution.
The clustering of the continuous domain will be described by an \emph{assignment function} $\af: \Real^N \to \{1,2,\dots,K\}$.
Density $\rho$ implies conditional probability densities $\dSk(x):=\rho(x\,|\,\af(x)=k)$.
Feature points $f_p$ in cluster $S^k$ could be interpreted as a sample from conditional density $\dSk$. 

Then, the continuous clustering problem is to find an assignment function optimizing a clustering criteria. For example, we can analogously to~\eqref{eq:gini_criterion} define continuous Gini clustering criterion
\begin{equation}\label{eq:cont gini criterion}
\left(\parbox{13ex}{\centering\bf\small continuous \\ Gini criterion}\right)\quad\quad\quad\quad
\sum_k w_k \cdot G(\af,k), \quad\quad\quad\quad\quad\quad
\end{equation}
where $w_k$ is the probability to draw a point from $k$-th cluster and
\begin{equation} \label{eq:gini_impurity}
G(\af,k)\;:=\; 1-\int \dSk (x)^2 \df x.   
\end{equation}

In the next section we show that kernel K-means energy~\eqref{eq:Parzen_energy} can be approximated by continuous Gini-clustering criterion~\eqref{eq:cont gini criterion} for \rsmall/ kernels.}

\subsection{Kernel K-means and continuous Gini criterion} \label{sec:kKM=Gini}

To establish the connection between kernel clustering and the Gini criterion, let us first recall Monte-Carlo estimation~\cite{UAI:97}, which yields the following expectation-based approximation for a continuous function $g(x)$ and cluster $C\subset\Omega$:   
\begin{equation} \label{eq:Monte-Carlo}
\sum_{p \in C} g(f_p) \approx |C| \int g(x) \, \rho_C(x) \, \df x
\end{equation}
where $\rho_C$ is the ``true'' continuous density of features in cluster $C$. 
Using \eqref{eq:Monte-Carlo} for $C=S^k$ and $g(x)=\Pk(x|S^k)$, we can approximate
the kernel density formulation in \eqref{eq:Parzen_energy} by its expectation
\begin{equation} \label{eq:parzen_criterion} 
F(S) \;\approxutc \; - \sum_k |S^k|  \int \Pk(x|S^k) \, \dSk(x) \, \df x .
\end{equation} 
Note that partition $S=(S^1,\dots,S^K)$ is determined by dataset $\Omega$ and assignment function $\af$. 
{\NEW We also assume
\begin{equation} \label{eq:density_approx}
 \Pk(\cdot\,|S^k) \;\;\approx\;\; \dSk(\cdot).
\end{equation}
This is essentially an assumption on kernel bandwidth. That is, we assume that kernel bandwidth 
gives accurate density estimation. For shortness, we call such bandwidths \rsmall/, see Section \ref{sec:prob inter}.
Then \eqref{eq:parzen_criterion} reduces} to approximation 
\begin{equation} 
F(S) \;\; \approxutc \;\; - \sum_k  |S^k|\cdot \int \dSk(x)^2 \df x \;\; \equivutc \;\; \sum_k  |S^k|\cdot G(\af,k).
\label{eq:gini approx}
\end{equation}
Additional application of Monte-Carlo estimation $|S^k|/|\Omega|\approx w_k$ allows replacing set cardinality $|S^k|$ by probability $w_k$ of drawing a point from $S^k$. This results in continuous Gini clustering criterion~\eqref{eq:cont gini criterion}, which approximates~\eqref{eq:Parzen_energy} or~\eqref{eq:kKmeans3} up to an additive and positive multiplicative constants.

Next section proves that the continuous Gini criterion~\eqref{eq:cont gini criterion} has a similar bias observed by Breiman in the discrete case.

\subsection{Breiman's bias in continuous Gini criterion}  \label{app:gini}

This section extends Theorem~\ref{th:breiman} to continuous Gini criterion \eqref{eq:cont gini criterion}. 
Since Section~\ref{sec:kKM=Gini} has already established 
a close relation between continuous Gini criterion and kernel K-means for \rsmall/ bandwidth kernels, then 
Breiman's bias also applies to the latter. {\NEW For simplicity, we focus on $K=2$ as in Breiman's Theorem~\ref{th:breiman}.}

  \begin{theor} [\thebias/ in continuous case] \label{th:gini_bias}
For $K=2$ the continuous Gini clustering criterion \eqref{eq:cont gini criterion} achieves its optimal value at the partitioning of $\Real^N$ into regions  $$\af_1=\arg\max_x \rho(x) \quad\text{and}\quad \af_2=\Real^N \setminus \af_1.$$
\end{theor}
\begin{proof}
The statement follows from Lemma~\ref{prop:appendixD} below.
\end{proof} 

We denote mathematical expectation of function $z\;:\;\Omega\to{\mathcal R}^1$
$$\E z  \; := \int z(x)\rho(x) \df x.$$

Minimization of \eqref{eq:cont gini criterion} corresponds to maximization of the following objective function
\begin{equation} \label{eq:gini:objective}
    L(\af) \;\; := \;\; w \int \dS(x)^2 \df x \;+\; (1-w) \int \dSbar(x)^2 \df x
\end{equation}
where {\NEW the probability to draw a point from cluster $1$ is}
$$w \;\; := \;\; w_1 \;\; = \;\;  \int_{\af(x)=1} \rho(x) \df x=\E [\af(x)=1] $$
where $[\cdot]$ is the indicator function.
{\NEW Note that \emph{mixed joint density} 
$$ \rho(x,k) \; = \; \rho(x) \cdot [s(x)=k] $$
allows to write conditional density $\dS$ in \eqref{eq:gini:objective} as
\begin{equation}    \label{eq:gini:conditional}
    \dS(x)\;=\; \frac{\rho(x,1)}{P(s(x) = 1)} \; = \; \rho(x)\cdot \frac{[\af(x)=1]}{w}.
\end{equation} }
 Equations \eqref{eq:gini:objective} and \eqref{eq:gini:conditional} give
\begin{align}
L(\af) \;\; =\;\; &\frac1{w}\int \rho(x)^2[\af(x)=1] \df x \notag \\
& + \; \frac1{1-w}\int \rho(x)^2[\af(x)=2]\df x.\end{align}
Introducing notation 
$$I \;\; := \;\; [\af(x)=1] \;\;\;\;\;\;\mbox{and}\;\;\;\;\;\; \rho \;\; := \;\; \rho(x)$$ 
allows to further rewrite objective function $L(\af)$ as
\begin{equation} \label{eq:l}
L(\af) \;\;=\;\; \frac{\E I\rho}{\E I} \; + \; \frac{\E (1- I)\rho}{1-\E I}.
\end{equation}

Without loss of generality assume that
$\frac{\E (1- I)\rho}{1-\E I} \le \frac{\E I \rho}{\E I}$ (the opposite case would yield a similar result). We now need following 
\begin{lem}\label{lm:ratio-law} 
Let $a,b,c,d$ be some positive numbers, then $$\frac ab \le \frac cd \implies \frac ab \le \frac{a+c}{b+d} \le \frac cd.$$
\begin{proof} Use reduction to a common denominator. \end{proof}
\end{lem}
\noindent Lemma \ref{lm:ratio-law} implies inequality
\begin{equation} \label{eq:gini:ratio}
\frac{\E (1- I)\rho}{1-\E I} \;\; \le \;\; \E\rho \;\; \le \;\; \frac{\E I\rho}{\E I},
\end{equation}
which is needed to prove the Lemma below.

\begin{lem}   \label{prop:appendixD}
Assume that function $\af_\varepsilon$ is
\begin{equation} \label{eq:se}
\af_\varepsilon(x) \;\; := \;\; \begin{cases}1, & \rho(x) \ge \sup_x \rho(x)-\varepsilon, 
\\ 
2, & \text{otherwise}. \end{cases}
\end{equation}
Then
\begin{equation} \label{eq:gini:main}
\sup_\af L(\af)=\lim_{\varepsilon\to0} L(\af_\varepsilon)=\E \rho + \sup_x \rho(x).
\end{equation}
 
\begin{proof} Due to monotonicity of expectation we have
\begin{align}
&  \frac{\E I \rho}{\E I} \le \frac{\E \left(I \sup_x \rho(x)\right)}{\E I} =
 \sup_x \rho(x). \label{eq:gini:up-incr}
\end{align}
Then \eqref{eq:gini:ratio} and \eqref{eq:gini:up-incr} imply
\begin{align}
L(\af) &= \frac{\E I\rho}{\E I} + \frac{\E (1- I)\rho}{1-\E I} \le \sup_x \rho(x)+\E \rho. \label{eq:gini:bound}
\end{align}
That is, the right part of \eqref{eq:gini:main} is an upper bound for $L(\af)$.

Let $I_\varepsilon\equiv[\af_\varepsilon(x)=1]$. It is easy to check that
\begin{equation} \label{eq:lim1}
\lim_{\varepsilon\to0}\frac{\E (1- I_\varepsilon) \rho}{1-\E I_\varepsilon}=\E \rho.
\end{equation}
Definition \eqref{eq:se} also implies 
\begin{equation} \lim_{\varepsilon\to0}\frac{\E I_\varepsilon \rho}{\E I_\varepsilon} \ge
 \lim_{\varepsilon\to0}\frac{\E (\sup_x \rho(x)-\varepsilon)  I_\varepsilon}{\E I_\varepsilon} =\sup_x \rho(x).\end{equation} 
This result and \eqref{eq:gini:up-incr} conclude that 
\begin{equation} \label{eq:lim2}
\lim_{\varepsilon\to0}\frac{\E I_\varepsilon \rho}{\E I_\varepsilon}=\sup_x \rho(x).
\end{equation}
Finally, the limits in \eqref{eq:lim1} and \eqref{eq:lim2} imply 
\begin{align}
\lim_{\varepsilon\to0} L(\af_\varepsilon)&=  \lim_{\varepsilon\to0} \frac{\E (1- I_\varepsilon) \rho}{1-\E I_\varepsilon}+\lim_{\varepsilon\to0} \frac{\E I_\varepsilon \rho}{\E I_\varepsilon}  \notag \\
&= \E \rho + \sup_x \rho(x).
\end{align}
This equality and bound \eqref{eq:gini:bound} prove \eqref{eq:gini:main}.
\end{proof}
\end{lem}

{  This result states that the optimal assignment function separates the mode of the density function from the rest of the data. The proof considers case $K=2$ for continuous Gini criterion approximating kernel K-means for \rsmall/ kernels. 
{\NEW The multi-cluster version for $K>2$ also has \thebias/. Indeed, it is easy to show that any two clusters in the optimal solution shall give optimum of objective~\eqref{eq:cont gini criterion}. Then, these two clusters are also subject to \thebias/. See a multi-cluster example in Figure~\ref{fig:labelme}.}

{\NEW \textbf{ Practical considerations:} While Theorem  \ref{th:gini_bias} suggests that the isolated density mode should be
a single point, in practice \thebias/ in kernel k-means isolates a slightly wider cluster around the mode, 
see Figures~\ref{fig:bias example}, \ref{fig:labelme}, \ref{fig:nash}(a-d), \ref{fig:grabcutexamples}.
Indeed, \thebias/ holds for kernel k-means when the assumptions in 
Section~\ref{sec:kKM=Gini} are valid. In practice, shrinking of the clusters invalidates 
approximations~\eqref{eq:parzen_criterion} and~\eqref{eq:density_approx} preventing the collapse of the clusters. }

\subsection{\NEW Connection to maximal cliques and dominant sets}\label{sec:dsets}

Interestingly, there is also a relation between {\em maximum cliques} and {\em density modes}. 
Assume $0$-\!$1$ kernel $[\|x-y\| \le \sigma]$ with bandwidth $\sigma$. Then, kernel 
matrix $A$ is a connectivity matrix corresponding to a $\sigma$\emph{-disk graph}. 
Intuitively, the maximum clique on this graph should be inside a disk 
with the largest number of points in it, which corresponds to the density mode.

Formally, mode isolation bias can be linked to both maximum clique and its weighted-graph generalization, {\em dominant set} \cite{Marcello2007}. It is known that
maximum clique \cite{motzkin1965maxima} and {\em dominant set} \cite{Marcello2007}
solve a two-region clustering problem with energy
\begin{equation} \label{eq:dset}
-\frac{\sum_{pq \in S^1} A_{pq}}{|S^1|}
\end{equation}
corresponding to average association~\eqref{eq:AA} for $K=1$ and $S^1 \subseteq \Omega$. 
Under the same assumptions as above, Gini impurity~\eqref{eq:gini_impurity} can be used as an approximation
reducing objective \eqref{eq:dset} to
\begin{equation} \label{eq:dsetl}
\frac{\E I\rho}{\E I}.
\end{equation}
Using \eqref{eq:gini:up-incr} and \eqref{eq:lim2} we can conclude that the optimum of 
\eqref{eq:dsetl} isolates the mode of density function~$\rho$. 
Thus, clustering minimizing \eqref{eq:dset} for \rsmall/ bandwidths also has \thebias/.
That is, for such bandwidths the concepts of maximum clique and dominant set for graphs 
correspond to the concept of {\em mode isolation} for data densities.
Dominant sets for the examples in Figures~\ref{fig:synexp}(c), \ref{fig:bias example}(a), and \ref{fig:nash}(d)
would be similar to the shown mode-isolating solutions. 
}

\begin{figure}[t]
\centering

\begin{tabular}{@{}c@{}c@{}}
\includegraphics[width=0.49\columnwidth,trim=3cm 5mm 3cm 3cm,clip]{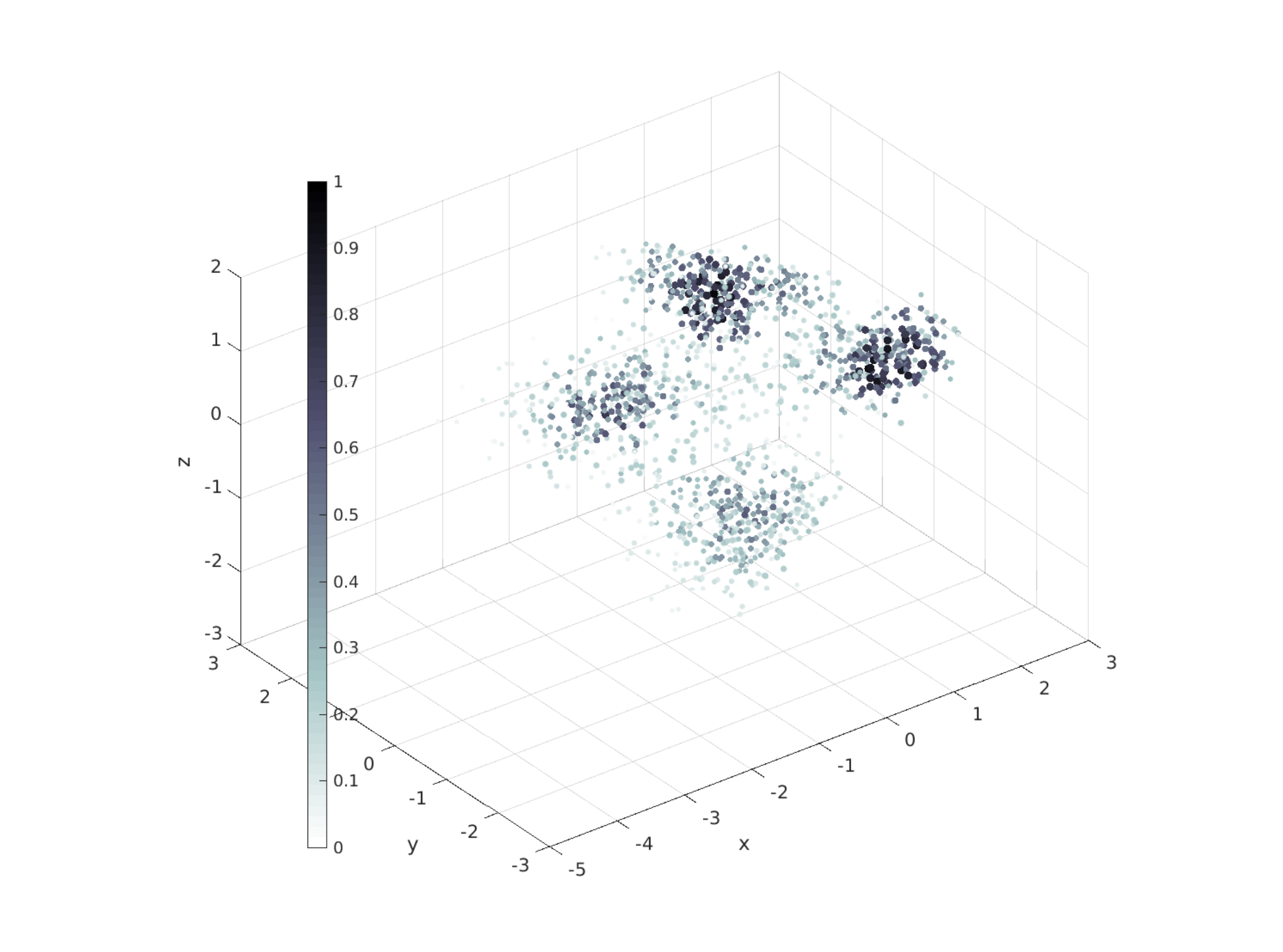} &
\includegraphics[width=0.49\columnwidth,trim=3cm 5mm 3cm 3cm,clip]{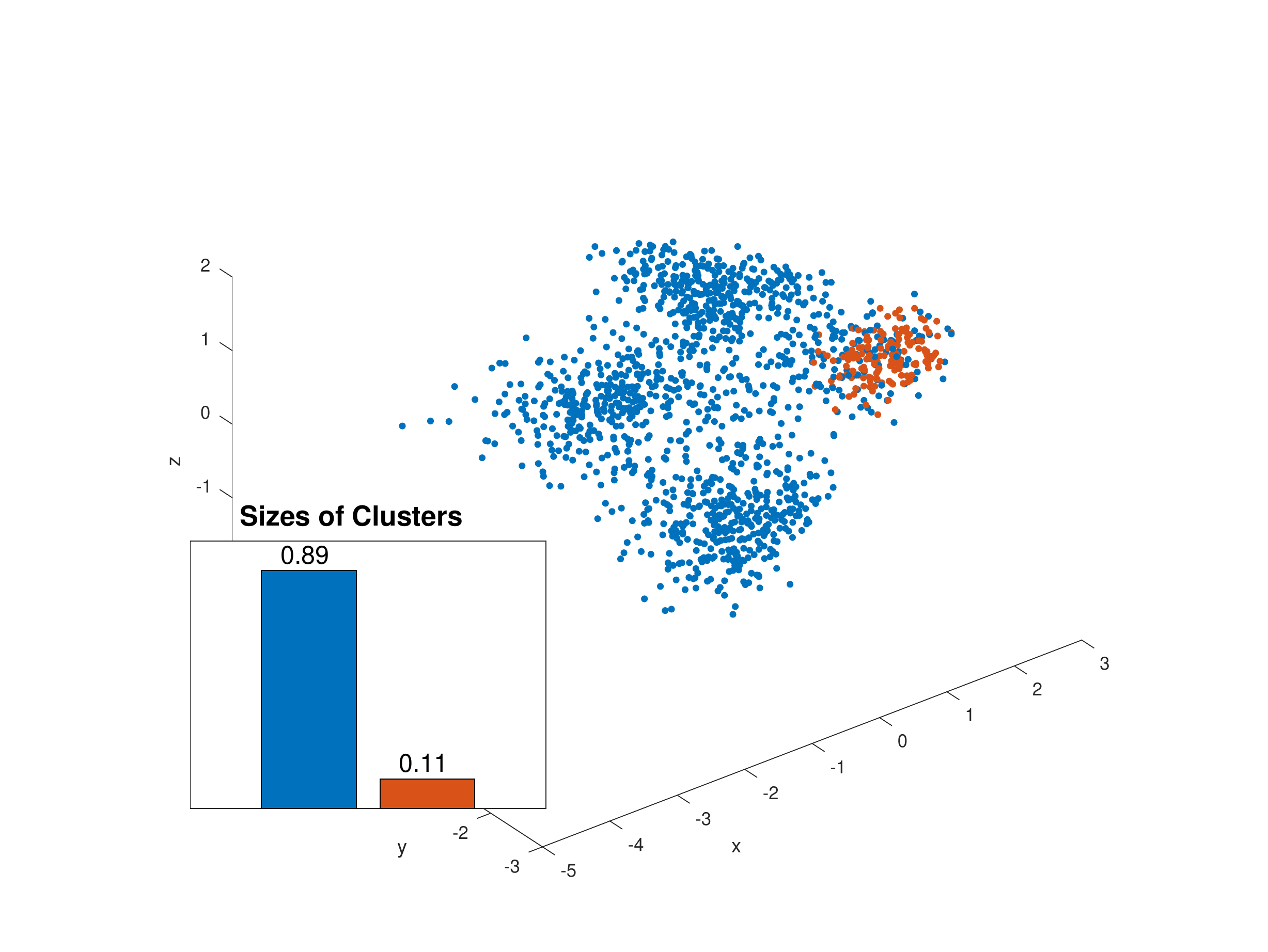}\\[-1ex]
(a) density & (b) Gaussian kernel, 2 clusters \\
\includegraphics[width=0.49\columnwidth,trim=3cm 5mm 3cm 3cm,clip]{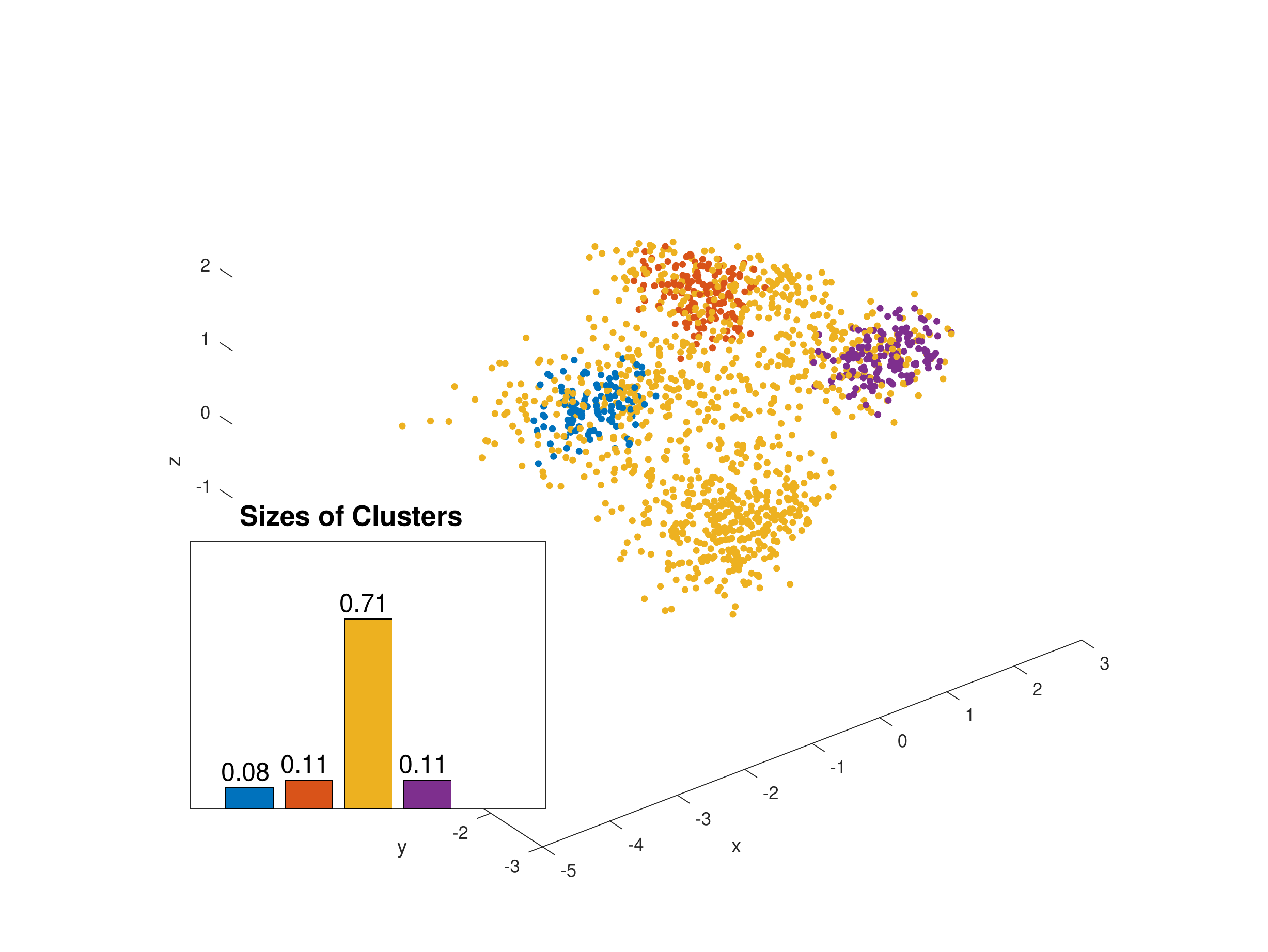} &
\includegraphics[width=0.49\columnwidth,trim=3cm 5mm 3cm 3cm,clip]{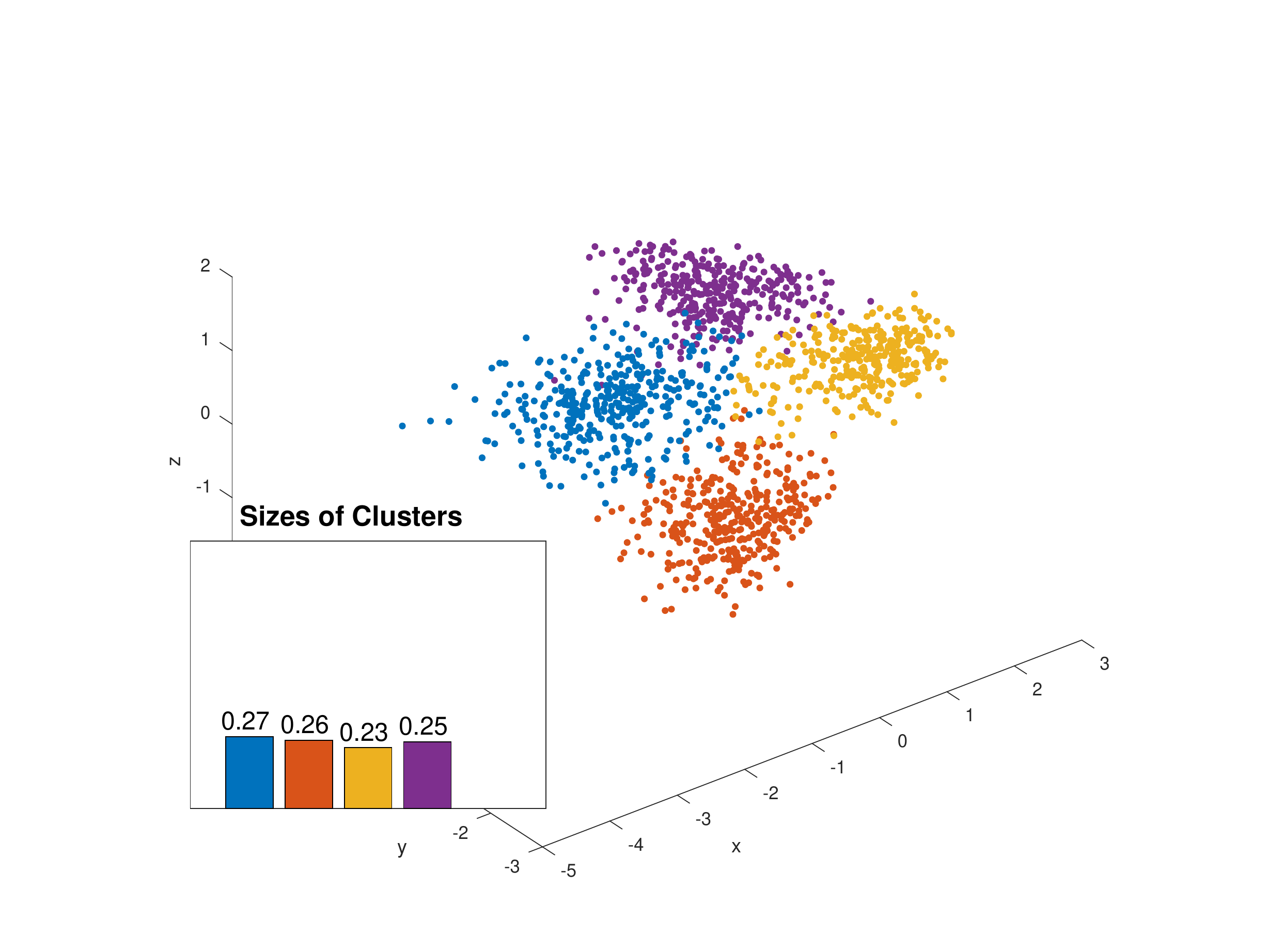} \\[-1ex]
(c) Gaussian kernel, 4 clusters & (d) \KNN/ kernel, 4 clusters \\[-1ex]
\end{tabular}
\caption{
\thebias/ in clustering of images. We select 4 categories from the LabelMe dataset \cite{oliva2001modeling}. 
The last fully connected layer of the neural network in \cite{krizhevsky2012imagenet} gives 4096-dimensional feature vector for each image. 
We reduce the dimension to 5 via PCA. For visualization purposes, we obtain 3D embeddings via MDS~\cite{cox2000mds}.
(a) Kernel densities estimates for data points are color-coded: darker points correspond to higher density. 
(b,c) The result of the kernel K-means with the Gaussian kernel \eqref{eq:gauss_k}. Scott's rule of thumb defines the bandwidth. 
\thebias/ causes poor clustering, \ie small cluster is formed in the densest part of the data in (b), 
three clusters occupy few points within densest regions while the fourth cluster contains 71\% of the data in (c). 
The \emph{normalized mutual information} (NMI) in (c) is 0.38.
(d) Good clustering produced by \KNN/ kernel $u_p$ (Example \ref{ex:KNN kernel}) gives NMI of 0.90, which is slightly better than the basic K-means (0.89).
}
\label{fig:labelme}
\end{figure}

\section{Adaptive weights solving \thebias/} \label{sec:adaptive weighting}

{

We can use a simple modification of average association by introducing weights $w_p \geq 0$ for each point ``error'' within the equivalent 
kernel K-means objective \eqref{eq:kmixed}
\begin{equation} \label{eq:wKKM}
F_w(S, m) \;\;=\;\; \sum_k \sum_{p \in S^k} w_p \|\phi_p - m_k \|^2.
\end{equation}
Such weighting is common for K-means \cite{duda1973pattern2}.
Similarly to Section \ref{sec:kkm overview} we can expand the Euclidean distances in \eqref{eq:wKKM} to obtain
an equivalent {\em weighted average association} criterion generalizing \eqref{eq:AA}
\begin{equation} \label{eq:wAA}
-\sum_k \frac{\sum_{pq\in S_k} w_p w_q A_{pq}}{\sum_{p\in S_k} w_p}.
\end{equation}
Weights $w_p$  have an obvious interpretation based on \eqref{eq:wKKM}; they change the data by replicating each point $p$ 
by a number of points in the same location (Figure~\ref{fig:density equalization}a) in proportion to $w_p$. 
Therefore, this weighted formulation directly modifies the data density as
\begin{equation} \label{eq:wDensity}
\rho'_p \propto w_p \rho_p
\end{equation}
where $\rho_p$ and $\rho'_p$ are respectively the densities of the original and the new (replicated) points. 
The choice of $w_p = 1/\rho_p$ is a simple way for equalizing data density to solve Breiman's bias. 
As shown in Figure~\ref{fig:density equalization}(a), 
such a choice enables low-density points to be replicated more frequently than high-density ones. 
This is one of density equalization approaches giving the solution in Figure \ref{fig:synexp}(d).

\begin{figure}
        \centering
        \begin{tabular}{c|c}
        \includegraphics[height=0.05in]{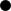} {\tiny - original data} \hspace{2ex}
         \includegraphics[height=0.05in]{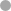} {\tiny - replicated data} & 
          \includegraphics[height=0.05in]{gini_figures/black.png} {\tiny - original data} \hspace{2ex}
         \includegraphics[height=0.05in]{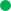} {\tiny - transformed data} \\
	  \includegraphics[height=0.8in]{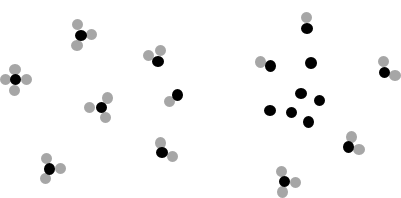} &
	  \includegraphics[height=0.8in]{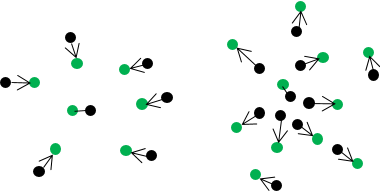} \\
         (a) adaptive weights (Sec.~\ref{sec:adaptive weighting})  & (b) adaptive kernels (Sec.~\ref{sec:density_equalization})
        \end{tabular}
        \caption{\label{fig:density equalization} {\em Density equalization} via (a) adaptive weights 
        and (b) adaptive kernels. In (a) the density is modified as in \eqref{eq:wDensity} via ``replicating'' each data point
        inverse-proportionately to the observed density using $w_p\propto 1/\rho_p$. 
        For simplicity (a) assumes positive integer weights $w_p$. 
        In (b) the density is modified according to \eqref{eq:density-from-bandwidth} for bandwidth \eqref{eq:knn bandwidth} 
        via implicit embedding of data points in a higher dimensional space that changes their relative positions. 
        }
\end{figure}

}

\section{Adaptive kernels solving \thebias/}  \label{sec:Bandwidth selection}

{ 
\thebias/ in kernel K-means is specific to \rsmall/ bandwidths. Thus, it has direct implications for the bandwidth selection problem discussed in this section.
\uline{Note that kernel bandwidth selection for \emph{clustering} should not be confused with kernel bandwidth selection for \emph{density estimation}, 
an entirely different problem outlined in Section~\ref{sec:prob inter}.}
In fact, \rsmall/ bandwidths give accurate density estimation, but yield poor clustering due to \thebias/. 
Larger bandwidths can avoid this bias in clustering. However, Section~\ref{sec:extreme} shows 
that for extremely large bandwidths kernel K-means reduces to standard K-means, 
which loses ability of non-linear cluster separation and has a different bias to equal cardinality clusters~\cite{UAI:97,volbias:ICCV15}.

In practice, avoiding extreme bandwidths is problematic since the notions of \emph{small} and \emph{large} strongly depend on data properties
that may significantly vary across the domain, {\NEW e.g. in Figure~\ref{fig:synexp}c,d where no fixed bandwidth gives a reasonable separation.} This motivates {\em locally} adaptive strategies.
Interestingly, Section~\ref{sec:Nash} shows that any locally adaptive bandwidth strategy implicitly corresponds to some data embedding
$\Omega \to \Real^{N'}$  deforming density of the points. That is, locally adaptive selection of bandwidth is equivalent to selection of density transformation.
Local kernel bandwidth and transformed density are related via the {\em density law} established in \eqref{eq:adaptive_sigma}.
As we already know from Theorem~\ref{th:gini_bias}, \thebias/ is caused by high non-uniformity of the data, which
can be addressed by density equalizing transformations. Section~\ref{sec:density_equalization} proposes 
adaptive kernel strategies based on our {\em density law} and motivated by a {\em density equalization} principle addressing \thebias/.
In fact, a popular locally adaptive kernel in \cite{zelnik2004self} is a special case of our density equalization principle.
} 

\subsection{Overview of extreme bandwidth cases} \label{sec:extreme}

{  Section \ref{sec:kKM=Gini} and Theorem~\ref{th:gini_bias} prove that for \rsmall/ bandwidths  
the kernel K-means is biased toward ``tight'' clusters, as illustrated in Figures~\ref{fig:synexp}, 
\ref{fig:bias example} and \ref{fig:nash}(d). 
As bandwidth increases, continuous kernel density~\eqref{eq:Parzen_estimate} no longer approximates the true distribution $\dSk$
violating \eqref{eq:density_approx}. Thus, Gini criterion~\eqref{eq:gini approx} is no longer 
valid as an approximation for kernel K-means objective~\eqref{eq:Parzen_energy}. In practice, \thebias/ disappears gradually 
as bandwidth gets larger. This is also consistent with experimental comparison of smaller and larger bandwidths in \cite{Shi2000}. 

The other extreme case of bandwidth for kernel K-means comes from its reduction to basic K-means for large kernels. For simplicity, 
assume Gaussian kernels \eqref{eq:gauss_k}  of large bandwidth $\sigma$ approaching 
data diameter. Then the kernel can be approximated by its Taylor expansion 
$\exp \left(-\frac{\|x-y\|^2}{2\sigma^2} \right) \approx 1-\frac{\|x-y\|^2}{2\sigma^2} $ and
kernel K-means objective~\eqref{eq:kKmeans3} for $\sigma \gg \|x-y\| $ becomes\footnote{Relation \eqref{eq:equivalence} easily follows 
by substituting $m_k\equiv \frac{1}{|S^k|}\sum_{p\in S^k} f_p$.} (up to a constant)
\begin{equation}    \label{eq:equivalence}
      \sum_k \frac{\sum_{pq \in S^k} \|f_p-f_q\|^2}{2\sigma^2 |S^k|}   \;\;\utc\;\; \frac{1}{\sigma^2} \;\; \sum_k \sum_{p \in S^k} \|f_p - m_k \|^2 , 
\end{equation}  
which is equivalent
to basic K-means~\eqref{eq:kmeans} for any fixed $\sigma$. 
}

\begin{figure}[t]

\setlength{\fboxsep}{0pt}

\begin{tikzpicture}[scale=0.96]
  {
	  \color{lightgray}
	  \draw[stealth-] (0.2in,0) -- (0.4in,0.25in);
	  \draw[stealth-] (0.7in,0) -- (1.3in,0.25in);
	  \draw[stealth-] (2.1in,0) -- (2.2in,0.25in);
	  \draw[stealth-] (3in,0) -- (3.1in,0.25in);
	  \node (A) at (0.4in,0.5in) {\fbox{\includegraphics[height=0.6in,trim=2cm 2cm 2cm 6mm,clip=true]{gini_figures/gaussian_uneven.pdf}}};
	  \node at (1.3in,0.5in) {\fbox{\includegraphics[height=0.6in,trim=2cm 2cm 2cm 6mm,clip=true]{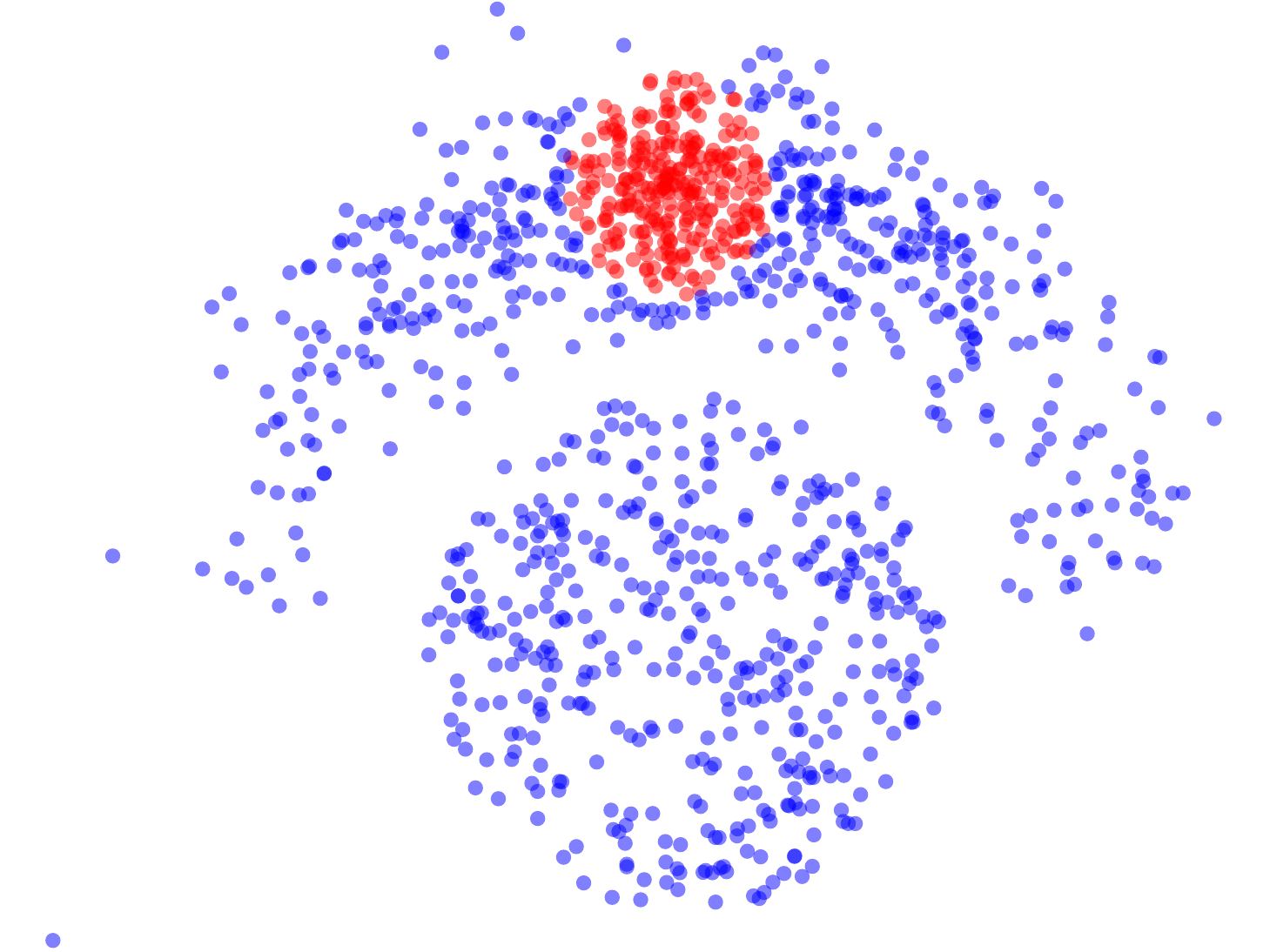}}};
	  \node at (2.2in,0.5in) {\fbox{\includegraphics[height=0.6in,trim=2cm 2cm 2cm 6mm,clip=true]{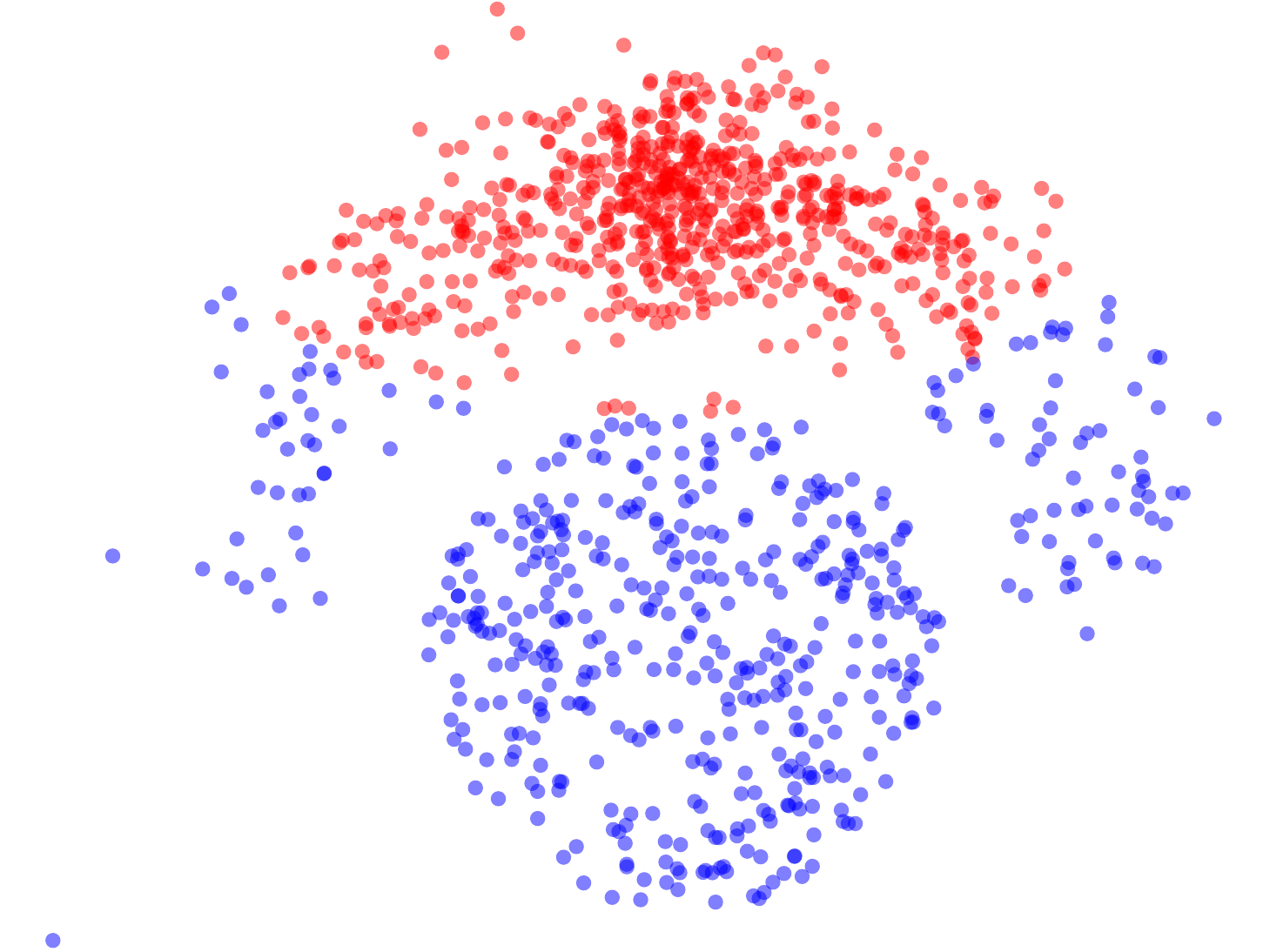}}};
	  \node at (3.1in,0.5in) {\fbox{\includegraphics[height=0.6in,trim=2cm 2cm 2cm 6mm,clip=true]{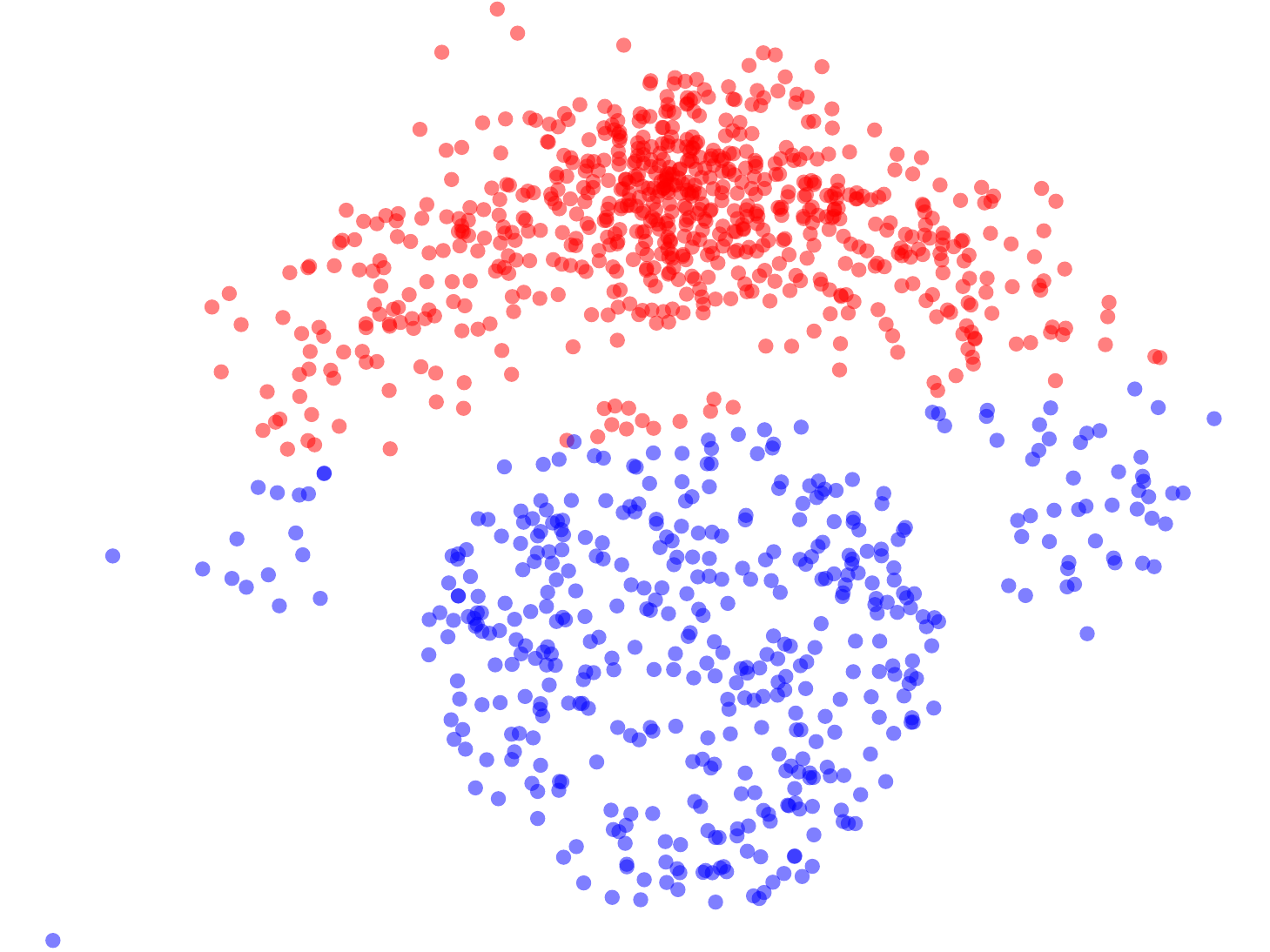}}}; 
  }
  \node at (0.2ex,-2.2ex) {$0$};
  \node at (\columnwidth,-2ex)  {$\infty$};
  \node at (\columnwidth-0.5ex,1in) [left]   { 
  		\parbox{25ex}{ \centering\bf ``equi-cardinality'' bias \\{\footnotesize (lack of non-linear separation)}}};
  		
  \node at (8ex,-2ex) {\rsmall/ $\sigma$};
  \node at (0.85in,1in)  {
  		\parbox{16ex}{ \centering\bf \thebias/ \\ {\footnotesize(mode isolation)}}};
  		
  \draw[|->,line width=1.1pt] (0,0) -- (\columnwidth,0) ;  
  \draw[line width=1.1pt] (0.7*\columnwidth,-0.5ex) -- (0.7*\columnwidth,0.5ex) ;  
  \node at (0.7*\columnwidth,-2ex) []   {$d_\Omega$};
  \foreach \x in {0.03, 0.045, ..., 0.28}{
  	\draw [opacity=(1 - ((\x-0.15)*7)^2)] (\x*\columnwidth, 0) -- ( \x*\columnwidth + 0.015*\columnwidth, 0.015*\columnwidth);
  }
  \foreach \x in {0.55, 0.565, ..., 0.98}{
  	\draw [opacity=((\x-0.55)*5)] (\x*\columnwidth, 0) -- ( \x*\columnwidth + 0.015*\columnwidth, 0.015*\columnwidth);
  }
\end{tikzpicture}
\caption{   Kernel K-means biases over the range of bandwidth $\sigma$. Data diameter is denoted by
$d_\Omega=\max_{pq\in\Omega}\|f_p-f_q\|$. \thebias/ is established for \rsmall/ $\sigma$ (Section \ref{sec:prob inter}).
Points stop interacting for $\sigma$ smaller than \rsmall/ making kernel K-means fail.
Larger $\sigma$ reduce kernel K-means to the basic K-means removing an ability to separate the clusters non-linearly.
In practice, there could be no intermediate good $\sigma$.
In the example of Fig.\ref{fig:synexp}(c) any fixed $\sigma$ leads to either \thebias/ or to the lack of
non-linear separability.
}
\label{fig:bias in kkm}
\end{figure}

{ 
Figure~\ref{fig:bias in kkm} summarizes kernel K-means biases for different bandwidths. 
For large bandwidths the kernel K-means loses its ability to find non-linear cluster separation due to reduction to the basic K-means.
Moreover,  it inherits the bias to equal cardinality clusters, which is well-known for the basic K-means  \cite{UAI:97,volbias:ICCV15}.
On the other hand, for small bandwidths kernel K-means has \thebias/ proven in Section \ref{sec:analysis}.
To avoid the biases in Figure~\ref{fig:bias in kkm}, kernel K-means should use a bandwidth neither too small nor too large. 
This motivates locally adaptive bandwidths.
}

\subsection{Adaptive kernels as density transformation} \label{sec:Nash}

{ 
This section shows that kernel clustering~\eqref{eq:kKmeans3} with any \emph{locally adaptive bandwidth} strategy
satisfying some reasonable assumptions is equivalent to \emph{fixed bandwidth} kernel clustering in a new feature space
(Theorem \ref{th:equiv_clustering}) with a deformed point density.  
The adaptive bandwidths relate to density transformations via {\em density law} \eqref{eq:adaptive_sigma}.
To derive it, we interpret {\em adaptiveness} as non-uniform variation of distances across the feature space. In particular, 
we use a general concept of {\em geodesic kernel} defining adaptiveness via a metric tensor 
and illustrate it by simple practical examples.

Our analysis of \thebias/  in Section~\ref{sec:analysis} applies to general kernels \eqref{eq:density kernel}
suitable for density estimation. Here we focus on clustering with kernels based on {\em radial basis functions} $\rbf$ s.t. 
\begin{equation} \label{eq:rbf}
\rbf(x-y)=\rbf(\|x-y\|).
\end{equation}
To obtain adaptive kernels, we replace Euclidean metric with Riemannian inside \eqref{eq:rbf}. 
In particular, $\|x-y\|$ is replaced with {\em geodesic distances} $d_g(x,y)$ between features $x,y \in {\cal R}^N$
based on any given metric tensor $g(f)$ for $f\in{\cal R}^N$. This allows to define a {\em geodesic} or 
{\em Riemannian} kernel at any points $f_p$ and $f_q$ as in \cite{Hartley:pami2015kernels}
\begin{equation} \label{eq:geo_kernel}
k_g(f_p,f_q)\;\;:=\;\;  \rbf  (d_g(f_p,f_q))  \;\; \equiv \;\;  \rbf (d_{pq})
\end{equation}
where  $d_{pq}:=d_g(f_p,f_q)$  is introduced for shortness.

In practice, the metric tensor can be defined only at the data points $g_p:= g(f_p)$ for $p\in\Omega$.
Often, quickly decaying radial basis functions $\psi$ allow Mahalanobis distance approximation inside \eqref{eq:geo_kernel}
\begin{equation} \label{eq:geodesic_dist}
d_g(f_p,x)^2  \;\; \approx \;\;   (f_p-x)^T g_p \, (f_p-x),
\end{equation} 
which is normally valid only in a small neighborhood of $f_p$. If necessary, one can 
use more accurate approximations for $d_g(f_p,f_q)$ based on Dijkstra \cite{cormen:2006} or Fast Marching method \cite{sethian1999level}.

\begin{example}[\textbf{\textit{Adaptive non-normalized\footnote{\NEW Lack of normalization as in \eqref{eq:non normalized adaptive}
is critical for {\em density equalization} resolving \thebias/, which is our only goal for adaptive kernels. Note that
without kernel normalization as in \eqref{eq:density kernel} Parzen density formulation of kernel k-means \eqref{eq:Parzen_energy} 
no longer holds invalidating the relation to Gini and \thebias/ in Section \ref{sec:analysis}. On the contrary, 
{\em normalized} variable kernels are appropriate for {\em density estimation}~\cite{TarrelScott92} 
validating \eqref{eq:Parzen_energy}. They can also make approximation \eqref{eq:density_approx} more accurate 
strengthening connections to Gini and \thebias/.}
Gaussian kernel}}]  \label{ex:non normalized adaptive}
Mahalanobis distances based on (adaptive) bandwidth matrices $\Sigma_p$ 
defined at each point $p$ can be used to define adaptive  kernel
\begin{equation}    \label{eq:non normalized adaptive}
\kappa_{p}(f_p,f_q) \; := \; \exp\frac{-(f_p-f_q)^T\Sigma^{-1}_{p}(f_p-f_q)}{2},
\end{equation}
which equals fixed bandwidth Gaussian kernel \eqref{eq:gauss_k} for $\Sigma_p = \sigma^2 I$.
Kernel \eqref{eq:non normalized adaptive} approximates \eqref{eq:geo_kernel}
for exponential function $\psi$ in \eqref{eq:normal kernel} and tensor $g$ continuously extending matrices 
$\Sigma^{-1}_p$ over the whole feature space so that $g_p = \Sigma^{-1}_p$ for $p\in\Omega$.
Indeed, assuming matrices $\Sigma^{-1}_p$ and tensor~$g$ change slowly between points 
within bandwidth neighbourhoods, one can use \eqref{eq:geodesic_dist} for all points in
\begin{equation} \label{eq:non normalized adaptive_approx}
\kappa_p(f_p,f_q)\;\;\approx\;\; \exp  \frac{-d_g(f_p,f_q)^2}{2}  \;\; \equiv \;\;\exp \frac{-d^2_{pq}}{2} 
\end{equation}
due to exponential decay outside the bandwidth neighbourhoods.
\end{example}

\begin{example}[\textbf{\textit{Zelnik-Manor \& Perona kernel \cite{zelnik2004self}}}]  \label{ex:Perona}
This popular  kernel is defined as $\kappa_{pq}:=\exp\frac{-\|f_p-f_q\|^2}{2\sigma_p \sigma_q}$. 
This kernel's relation to \eqref{eq:geo_kernel} is less intuitive due to the lack of ``local'' Riemannian tensor. 
However, under assumptions similar to those in \eqref{eq:non normalized adaptive_approx}, it can still be seen
as an approximation of geodesic kernel  \eqref{eq:geo_kernel} for some tensor $g$ such that $g_p=\sigma_p^{-2}I$ for $p\in\Omega$.
They use heuristic $\sigma_p=R_p^K$, which is the distance to the K-th nearest neighbour of~$f_p$.
 \end{example}

\begin{example}[\textbf{\textit{KNN kernel}}] \label{ex:KNN kernel}
This adaptive kernel is defined as $u_p(f_p,f_q)  = [f_q \in \KNN/(f_p)]$ 
where $\KNN/(f_p)$ is the set of $K$ nearest neighbors of $f_p$. This kernel approximates \eqref{eq:geo_kernel} for 
uniform function $\psi(t)=[t<1]$ and tensor $g$ such that $g_p=I/(R^K_p)^{2}$.
\end{example}

\begin{figure}[t]
\centering
\begin{tabular}{ cc }
{\small (a) space of points $f$}  & {\small (b) transformed points  $f'$}   \\
{\small with Riemannian metric $g$}  & {\small with Euclidean metric} \\
\begin{tikzpicture}[scale=1.2,color=black]
  \draw (-2ex,-2ex) rectangle (21ex, 10ex); 

  \begin{scope}[shift={(1ex,0)}]
  \draw[rotate=-30,color=blue!50,fill=black!10,thick,scale=0.7] (0,5ex) ellipse (4ex and 5ex) ;
  \draw[color=black,fill] (1.8ex,2.9ex) circle (0.5pt);
  \node at (3ex,4ex) {\footnotesize $g_1$};
  \end{scope}  
  
  \begin{scope}[shift={(4.5ex,0)}]
   \draw[rotate=20,color=blue!50,fill=black!10,thick] (14ex,0) ellipse (2ex and 4ex);
   \draw[color=black,fill] (13.3ex,4.7ex) circle (0.5pt);
   \node[above] at (13.3ex,4.7ex) {\footnotesize $g_2$};  
  \end{scope}  
  
  \begin{scope}[shift={(2.5ex,0)}]
   \draw[color=blue!50,fill=black!10,thick] (7ex,0ex) circle (1ex);
   \draw[color=black,fill] (7ex,0ex) circle (0.5pt);
   \node at (9.5ex,0.5ex) {\footnotesize $g_3$};  
  \end{scope}  
\end{tikzpicture} &
\begin{tikzpicture}[scale=1.2,color=black]
  \draw (-2ex,-2ex) rectangle (21ex, 10ex); 

  \node (A) at (2ex,5ex) {};
  \draw[color=blue!50,fill=black!10,thick] (A) circle (2ex) ;
  \draw[color=black,fill] (A) circle (0.5pt);
  \draw[->] (2ex,5ex) -- +(0.7ex*2,-0.7ex*2);
  \node[right] at (A) {\footnotesize $1$};
  
  \begin{scope}[shift={(4.5ex,0)}]
   \draw[color=blue!50,fill=black!10,thick] (13ex,7ex) circle (2ex) ;
   \draw[color=black,fill] (13ex,7ex) circle (0.5pt);
  \draw[->] (13ex,7ex) -- +(0.7ex*2,-0.7ex*2);
   \node[right] at (13ex,7ex) {\footnotesize $1$};
  \end{scope}  
  
  \begin{scope}[shift={(2.5ex,1ex)}]
   \draw[color=blue!50,fill=black!10,thick] (8ex,1ex) circle (2ex) ;
   \draw[color=black,fill] (8ex,1ex) circle (0.5pt);
  \draw[->] (8ex,1ex) -- +(0.7ex*2,-0.7ex*2);
   \node[right] at (8ex,1ex) {\footnotesize $1$};  
  \end{scope}  
\end{tikzpicture}
\\[-1ex]
{\footnotesize unit balls in Riemannian metric} &
{\footnotesize unit balls in Euclidean metric} 
\end{tabular}
\centering
 \caption{Adaptive kernel \eqref{eq:geo_kernel} based on Riemannian distances (a) is equivalent to
fixed bandwidth kernel after some {\em quasi-isometric} \eqref{eq:d tilde} embedding into Euclidean space (b), 
see Theorem~\ref{th:equiv_clustering}, mapping ellipsoids \eqref{eq:ellipsoid} to balls \eqref{eq:ball} and 
modifying data density as in \eqref{eq:density_transform}.\label{fig:ne}} 
\end{figure}
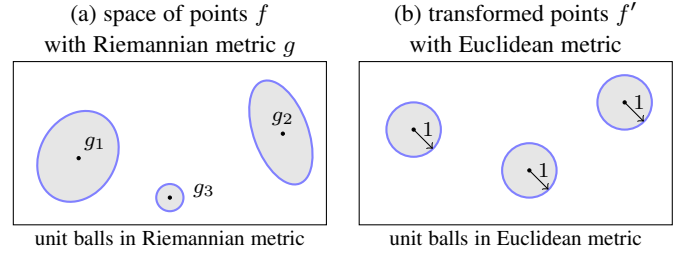

\begin{theor} \label{th:equiv_clustering}
Clustering \eqref{eq:kKmeans3} with (adaptive) geodesic kernel \eqref{eq:geo_kernel} is equivalent to
clustering with fixed bandwidth kernel $k'(f'_p,f'_q):=\psi'(\|f'_p-f'_q\|)$ in new feature space $\Real^{N'}$ for 
some radial basis function $\psi'$ using the Euclidean distance and some constant $N'$.
\end{theor}
\begin{proof}
A powerful general result in~\cite{lingoes1971some,gower1986metric,buhmann:pami03} states that for any symmetric 
matrix $(d_{pq})$ with zeros on the diagonal there is a constant $h$ such that squared distances 
\begin{equation} \label{eq:d tilde}
\widetilde d_{pq}^2\;\;=\;\;d^2_{pq} + h^2[p \ne q] 
\end{equation} 
form \emph{Euclidean matrix} $(\widetilde d_{pq})$. That is, there exists some Euclidean embedding 
$\Omega \to {\cal R}^{N'}$ where for  $\forall p\in \Omega$ there corresponds a point $f'_p \in {\cal R}^{N'}$ such that \mbox{$\|f'_p-f'_q\|=\widetilde d_{pq}$}, see Figure~\ref{fig:ne}.
Therefore,
\begin{equation} \label{eq:new_kernel} 
 \psi(d_{pq}) \;\;=\;\;  \psi \left(\sqrt{\widetilde{d}_{pq}^2 - h^2 \, [d_{pq} \geq h]} \right)  \;\; \equiv \;\;   \psi'(\widetilde{d}_{pq})
\end{equation} for $\psi'(t)\!:=\!\psi(\sqrt{t^2-h^2[t \geq h]})$ and $k_g(f_p,f_q)\!=\!k'(f'_p,f'_q)$. 
\end{proof}

Theorem \ref{th:equiv_clustering} proves that {\em adaptive} kernels for $\{f_p\}\subset{\cal R}^{N}$
can be equivalently replaced by a {\em fixed} bandwidth kernel for some implicit 
embedding\footnote{The implicit embedding implied by Euclidean matrix \eqref{eq:d tilde} 
should not be confused with embedding in the Mercer's theorem for kernel methods.} 
$\{f'_p\} \subset {\cal R}^{N'}$ in a new space. 
Below we establish a relation between three local properties at point $p\,$: 
adaptive bandwidth represented by matrix $g_p$ and two densities $\rho_p$ and $\rho'_p$ in the original and the new feature spaces.
For $\varepsilon>0$ consider an ellipsoid in the original space $\mathcal R^N$, see Figure~\ref{fig:ne}(a),
\begin{equation} \label{eq:ellipsoid}
B_p \;\;:=\;\; \{ x \; | \; (x-f_p)^T g_p\, (x-f_p) \le \varepsilon^2\}.
\end{equation} 
Assuming $\varepsilon$ is small enough so that approximation~\eqref{eq:geodesic_dist} holds,
ellipsoid \eqref{eq:ellipsoid} covers features $\{f_q\,|\,q\in\Omega_p\}$ for subset of points
\begin{equation}
\Omega_p \;:=\;\; \{q\in\Omega \;|\; d_{pq} \leq\varepsilon \}.
\end{equation}
Similarly, consider a ball in the new space $\mathcal R^{N'}$, see Figure~\ref{fig:ne}(b),
\begin{equation} \label{eq:ball}
B'_p \;\;:=\;\; \{ x \; | \; \|x-f'_p\|^2 \le \varepsilon^2 + h^2\}
\end{equation} 
covering features $\{f'_q\,|\,q\in\Omega'_p\}$ for points
\begin{equation}
\Omega'_p \;\;:=\;\; \{q\in\Omega \;|\; \widetilde{d}^2_{pq} \leq\varepsilon^2 + h^2 \}.
\end{equation}

It is easy to see that \eqref{eq:d tilde} implies $\Omega_p=\Omega'_p$.
Let $\rho_p$ and $\rho'_p$ be the densities\footnote{We use the physical rather than probability density. 
They differ by a factor.} of points within $B_p$ and $B'_p$ correspondingly. Assuming
$|\cdot|$ denotes volumes or cardinalities of sets, we have
\begin{equation} \label{eq:density derivation}
\rho_p\cdot |B_p| \;\;=\;\; |\Omega_p| \;\;=\;\; |\Omega'_p| \;\;=\;\;\rho'_p\cdot |B'_p|.
\end{equation}   }
Omitting a constant factor depending on $\varepsilon$, $h$, $N$ and $N'$ we get
\begin{equation} \label{eq:density_transform}
\rho'_p \;\;=\;\; \rho_p\; \frac{|B_p|}{|B'_p|}  \;\;\propto\;\; \rho_p \; |\!\det g_p|^{-\frac{1}{2}}
\end{equation}   
representing  the general form of the {\em density law}. 
For the basic isotropic metric tensor such that $g_p = I/\sigma_p^2$ it simplifies to
\begin{equation} \label{eq:density-from-bandwidth}
\rho'_p \;\;\propto\;\; \rho_p\, \sigma_p^N.
\end{equation} 
Thus, bandwidth $\sigma_p$ can be selected adaptively based on any desired transformation of 
density $\rho'_p \equiv \tau(\rho_p)$  using
\begin{equation} \label{eq:adaptive_sigma}
\sigma_{p} \;\; \propto \;\; \sqrt[N]{\tau(\rho_p)/\rho_p}.
\end{equation}
where observed density $\rho_p$ in the original feature space can be evaluated at any point $p$ 
using any standard estimators, \eg \eqref{eq:Parzen_estimate}.

\subsection{Density equalizing locally adaptive kernels} \label{sec:density_equalization}

\begingroup
\setlength{\columnsep}{1ex}%
\setlength{\intextsep}{-1ex}%
Bandwidth formula \eqref{eq:adaptive_sigma} works for any density transform~$\tau$. 
To address Breiman's bias, one can use density equalizing transforms $\tau(\rho)=\const$ or 
$\tau(\rho)=\frac{1}{\alpha}\log(1 + \alpha \rho)$, which even up 
\begin{wrapfigure}{r}{40mm}
\vspace{-5ex}
\begin{tikzpicture}[scale=0.8]
  \draw[->] (-0.1,0) -- (3.5,0);
  \node at (3.5,-.3) [left] {\small original density $\rho$};
  \draw[->] (0,-0.1) -- (0,2);
  \node at (-0.4,1) [text width=3cm, align=center,rotate=90] {\small new density};
  \draw[thick,domain=0:2,smooth,variable=\x,gray] plot ({\x},{abs(\x)});
  \draw[thick,domain=0:3.5,smooth,variable=\x,red] plot ({\x},{ln(1+\x)});
  \draw[thick,domain=0:3.5,smooth,variable=\x,blue] plot ({\x},{1});
  \node[right,gray] at (2,1.9) {\small $\tau(\rho)= \rho$};
  \node[right,red] at (0.5,0.3) {\small $\tau(\rho)=\frac1\alpha\log(1+\alpha \rho)$};
  \node[right,blue] at (1.5,0.8) {\small $\tau(\rho)=\const$};
\end{tikzpicture}
\end{wrapfigure}
the highly dense parts of the feature space as 
illustrated on the right. 
Some empirical results using density equalization $\tau(\rho)=\const$ for synthetic and real data are shown in Figures \ref{fig:synexp}(d)  and \ref{fig:nash}(e,f).

\endgroup

\begin{figure}
\centering
\footnotesize
\begin{tabular}{ cc|c } 
\multicolumn{2}{c|}{using \textbf{\textit{fixed}} width kernel}& using \textbf{\textit{adaptive}} kernel\\

\includegraphics[width=0.23\linewidth]{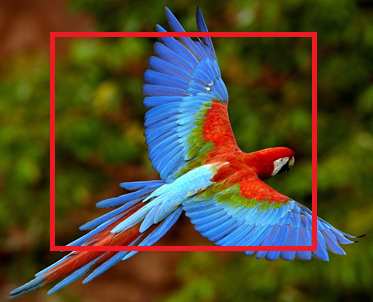} & \includegraphics[width=0.23\linewidth]{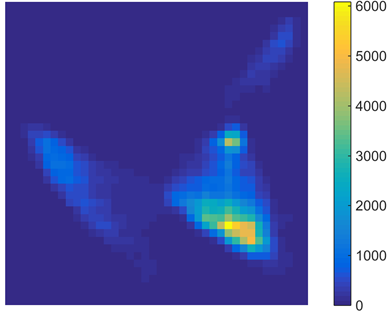} & \includegraphics[width=0.23\linewidth]{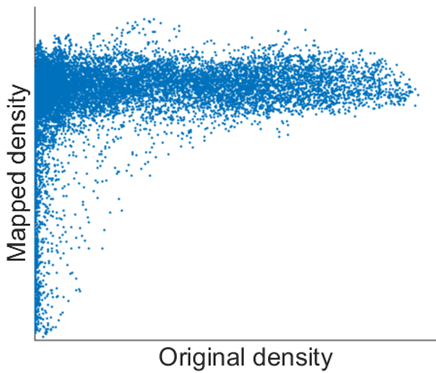} \\ 
(a) input image & (b) 2D color histogram & (e) density mapping \\ 
\includegraphics[width=0.23\linewidth]{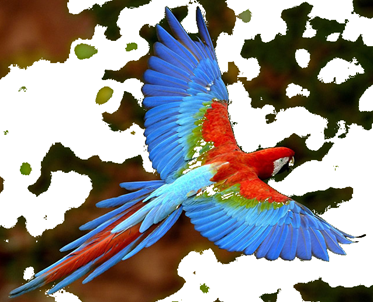} & \includegraphics[width=0.2\linewidth]{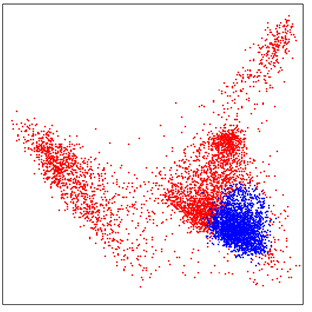} & \includegraphics[width=0.23\linewidth]{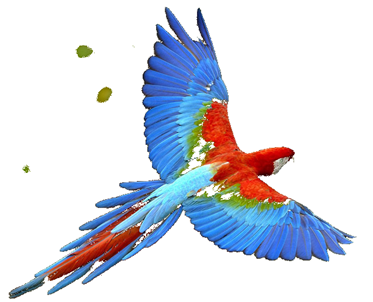} \\ 
(c) clustering result& (d) color coded result & (f) clustering result
\end{tabular}
\caption{(a)-(d): {\em \thebias/} for fixed bandwidth kernel \eqref{eq:gauss_k}. 
(f):~result for \eqref{eq:non normalized adaptive} with adaptive bandwidth \eqref{eq:knn bandwidth} s.t. $\tau(\rho)\!=\const$.
(e) {\em density equalization}: scatter plot of empirical densities in the original/new feature spaces
obtained via \eqref{eq:Parzen_estimate} and \eqref{eq:d tilde}.}
\label{fig:nash}
\end{figure}

One way to estimate the density in \eqref{eq:adaptive_sigma} is \KNN/ approach~\cite{Bishop06}
\begin{equation}
\rho_p \;\; \approx \;\; \frac{K}{n V_K} \;\; \propto \;\; \frac{K}{n (R_p^K)^N}
\end{equation} 
where $n\equiv|\Omega|$ is the size of the dataset, $R_p^K$ is the distance to the $K$-th nearest neighbor of $f_p$, 
$V_K$ is the volume of a ball of radius $R_p^K$ centered at $f_p$. 
Then, density law \eqref{eq:adaptive_sigma} for $\tau(\rho)=\const$ gives
\begin{equation} \label{eq:knn bandwidth}
\sigma_p \;\; \propto \;\; R_p^K 
\end{equation}
consistent with heuristic bandwidth in \cite{zelnik2004self}, see Example~\ref{ex:Perona}.

\begingroup
\setlength{\columnsep}{1ex}%
\setlength{\intextsep}{0ex}%

The result in Figure~\ref{fig:synexp}(d) uses adaptive Gaussian kernel \eqref{eq:non normalized adaptive}
for $\Sigma_p=\sigma_p I$ with $\sigma_p$ derived in~\eqref{eq:knn bandwidth}. 
Theorem~\ref{th:equiv_clustering} claims equivalence 
to a fixed bandwidth kernel in some transformed higher-dimensional space $\Real^{N'}\!$.
Bandwidths \eqref{eq:knn bandwidth} are chosen specifically 
to equalize the data density
in this space so that $\tau(\rho)=\const$. 
\begin{wrapfigure}{r}{40mm}
\centering 
{\includegraphics[width=40mm,trim=15mm 30mm 15mm 25mm,clip]{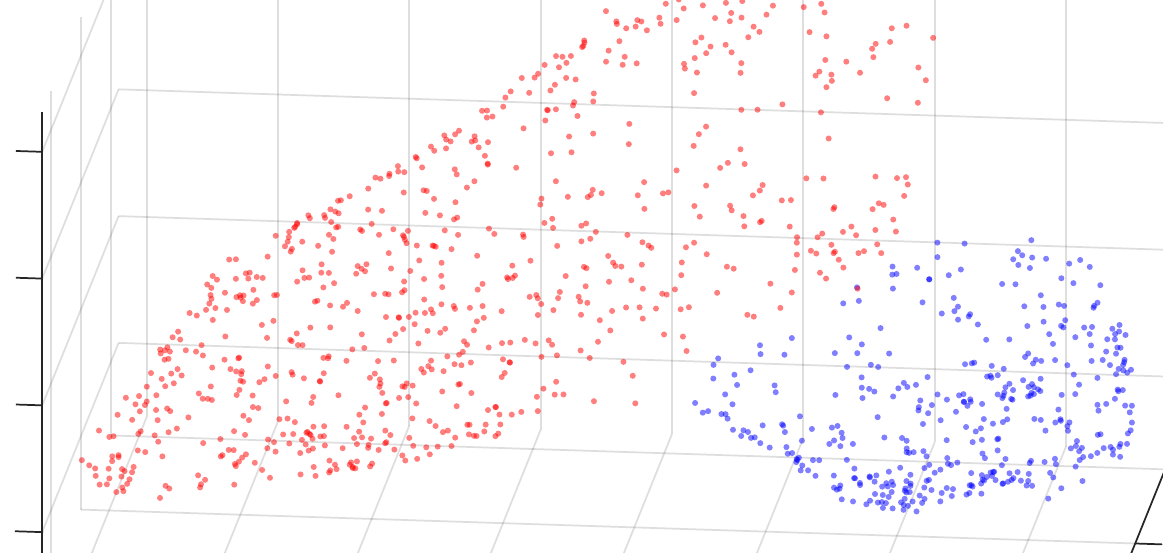}}
\end{wrapfigure}
The picture on the right illustrates such density equalization for the data in Figure~\ref{fig:synexp}(d).
It shows a 3D projection of the transformed data obtained by \emph{multi-dimensional scaling} \cite{cox2000mds} 
for matrix $(\widetilde{d}_{pq})$ in \eqref{eq:d tilde}.
The observed density equalization removes \thebias/ from the clustering in Figure~\ref{fig:synexp}(d).

\endgroup

{ 
Real data experiments for kernels with adaptive bandwidth  \eqref{eq:knn bandwidth}
are reported in Figures~\ref{fig:bias example}, \ref{fig:labelme}, \ref{fig:nash}, \ref{fig:grabcutexamples} and Table~\ref{tb:errorrates}. 
Figure \ref{fig:nash}(e) illustrates the empirical {\em density equalization} effect for this bandwidth. 
Such data homogenization removes the conditions leading to \thebias/, see Theorem~\ref{th:gini_bias}.
Also, we observe empirically that \KNN/ kernel is competitive with adaptive Gaussian kernels, 
but its sparsity gives efficiency and simplicity of implementation.}

\begin{figure}
        \centering
        \includegraphics[width=0.85\columnwidth,trim=1cm 0.15cm 9cm 0.15cm,clip]{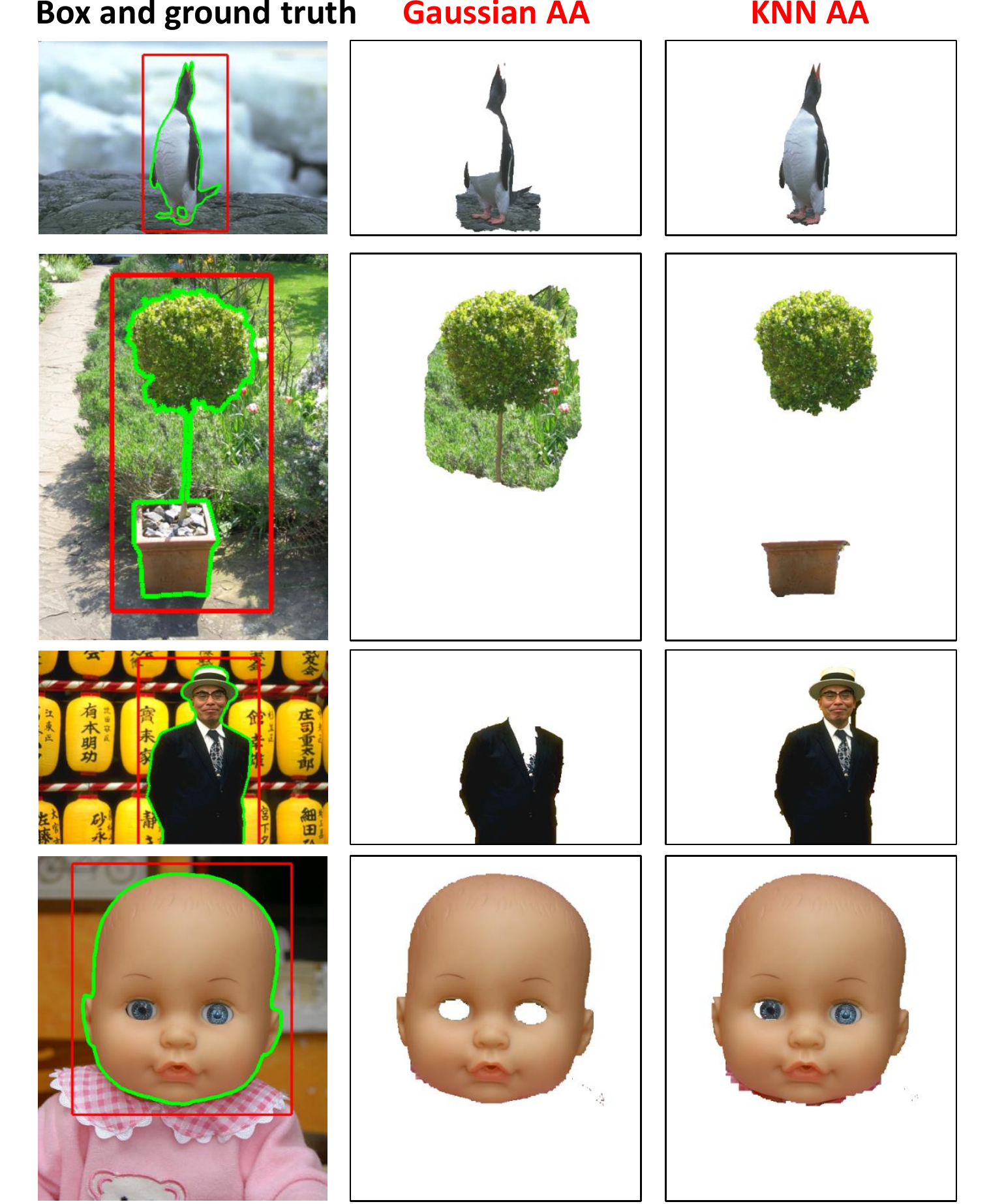}
         \caption{Representative interactive segmentation results. 
         Regularized average association (\AA/) with fixed bandwidth kernel \eqref{eq:gauss_k} 
         or adaptive \KNN/ kernels (Example \ref{ex:KNN kernel}) is optimized as in \cite{kkm:iccv15}. 
         Red boxes define initial clustering, green contours define ground-truth clustering. 
         Table \ref{tb:errorrates} provides the error statistics. \thebias/ manifests itself by isolating the 
         most frequent color from the rest. \label{fig:grabcutexamples}}
\end{figure}

\begin{table}
{\footnotesize
\hfill{}
 \begin{tabular}{ c | cccc}  \hline
{\bf regularization } & \multicolumn{4}{c}{  \bf average error, \% } \\ 
\begin{tabular}[c]{@{}c@{}}(boundary\\smoothness)\end{tabular}  &
\begin{tabular}[c]{@{}c@{}}Gaussian\\\AA/\end{tabular} & 
\begin{tabular}[c]{@{}c@{}}Gaussian\\\NC/\end{tabular}  &  
\begin{tabular}[c]{@{}c@{}}\KNN/\\\AA/\end{tabular}   &  
\begin{tabular}[c]{@{}c@{}}\KNN/\\\NC/\end{tabular}  \\ \hline
none$^\dag$                & 20.4  & 17.6 & \textbf{12.2} & {12.4} \\ \hline
Euclidean length$^*$   & 15.1  & 16.0 & \textbf{10.2} & {11.0} \\ \hline
contrast-sensitive$^*$  &   9.7  & 13.8 &  \textbf{7.1} & {7.8} \\  \hline
  \end{tabular}
\hfill{}}
\caption{Interactive segmentation errors. \AA/ stands for the average association, \NC/ stands for the normalized cut. 
Errors are averaged over the GrabCut dataset\cite{GrabCuts:SIGGRAPH04}, see samples in Figure~\ref{fig:grabcutexamples}.  
$^*$We use \cite{kkm:iccv15,nc+mrf:eccv16} for a combination of Kernel K-means objective~\eqref{eq:kKmeans3} with \emph{Markov Random Field} (MRF) regularization terms. 
The relative weight of the MRF terms is chosen to minimize the average error on the dataset. 
$^\dag$Without the MRF term, \cite{kkm:iccv15} and \cite{nc+mrf:eccv16} correspond to the standard kernel K-means~\cite{dhillon2004kernel,Chitta2011}.}
\label{tb:errorrates}
\end{table}

\section{Normalized Cut and \thebias/} \label{sec:NC and BB}

{

\thebias/ for kernel K-means criterion \eqref{eq:kKmeans3}, a.k.a. {\em average association} (\AA/)  \eqref{eq:AA}, was empirically 
identified in \cite{Shi2000}, but our Theorem~\ref{th:gini_bias} is its first theoretical explanation.
This bias was the main critique against \AA/ in \cite{Shi2000}. 
They also criticize {\em graph cut} \cite{WuLeahy:PAMI93} that ``favors cutting small sets of isolated nodes''.
These critiques are used to motivate {\em normalized cut} (\NC/) criterion \eqref{eq:NC} aiming at balanced 
clustering without ``clumping'' or ``splitting''.

We do not obeserve any evidence of the {\em mode isolation bias} in \NC/. However,
Section \ref{sec:NC sparse subsets} demonstrates that \NC/ still has a bias to isolating sparse subsets.
Moreover, using the general density analysis approach introduced in Section \ref{sec:Nash} we also show in Section \ref{sec:density inversion}
that {\em normalization} implicitly corresponds to some density-inverting embedding of the data. Thus,
{\em mode isolation} (\thebias/) in this implicit embedding corresponds to the {\em sparse subset bias} of \NC/ in the original data.

\subsection{Sparse subset bias in Normalized Cut} \label{sec:NC sparse subsets}

The normalization in \NC/ does not fully remove the bias to small isolated subsets
and it is easy to find examples of ``splitting'' for weakly connected nodes, see Figure~\ref{fig:nc fails}(a).
The motivation argument for the \NC/ objective 
below Fig.1 in \cite{Shi2000} implicitly assumes similarity matrices with zero diagonal, which excludes many common similarities like 
Gaussian kernel \eqref{eq:gauss_k}. Moreover, their argument is built specifically for an example with a single isolated point, 
while an isolated pair of points will have a near-zero \NC/ cost even for zero diagonal similarities.

Intuitively, this \NC/ issue can be interpreted as a bias to the ``sparsest'' subset (Figure~\ref{fig:nc fails}a), the opposite of \AA/'s bias to the ``densest'' subset, 
\ie \thebias/ (Figure~\ref{fig:synexp}c). The next subsection discusses the relation between these opposite biases in detail. In any case, 
both of these density inhomogeneity problems in \NC/ and \AA/ are directly addressed by our {\em density equalization} principle 
embodied in adaptive weights $w_p\propto 1/\rho_p$ in Section~\ref{sec:adaptive weighting} or in the locally adaptive kernels derived in Section \ref{sec:density_equalization}. 
Indeed, the result in Figure~\ref{fig:synexp}(d) can be replicated with \NC/ using such adaptive kernel.
Interestingly, \cite{zelnik2004self} observed another data non-homogeneity problem in \NC/ different from the sparse subset bias in Figure~\ref{fig:nc fails}(a), 
but suggested a similar adaptive kernel as a heuristic solving it.

\begin{figure}
        \centering
        \setlength\tabcolsep{2 pt}
        \begin{tabular}{cc}
         \begin{tikzpicture}{1.6in}
        \node at (0,0) {\includegraphics[width=1.6in]{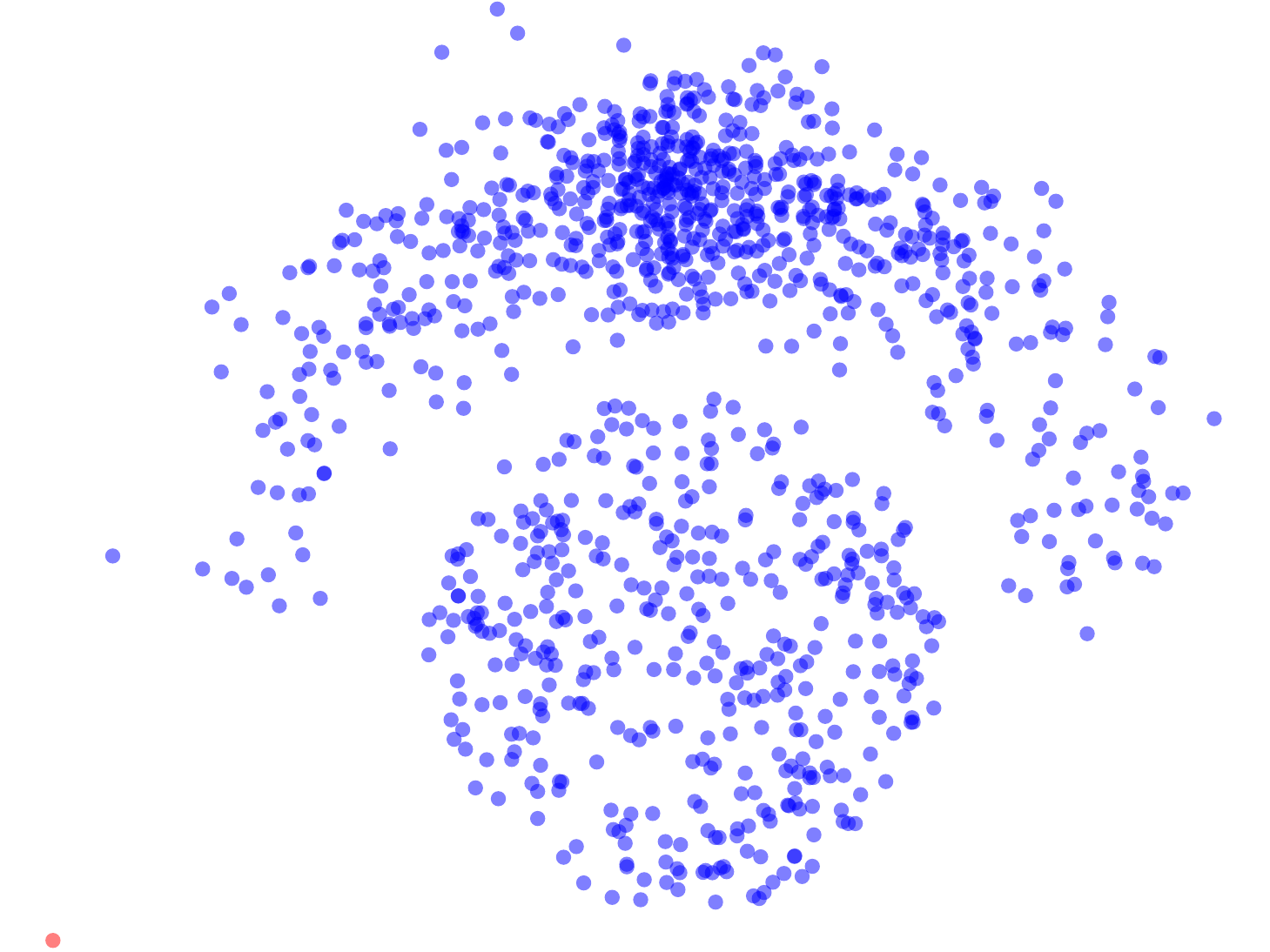}};
        \node at (0.55in,-0.45in) {\tiny $\begin{aligned}\sigma&=2.47 \\[-0.5ex] \NC/&=0.202\end{aligned}$};
        \draw[-stealth, line width=0.4mm, color=red] (-1.4,-1.02) -- +(-0.3,-0.3);
        \end{tikzpicture}  &
         \begin{tikzpicture}{1.6in}
        \node at (0,0) {\includegraphics[width=1.6in]{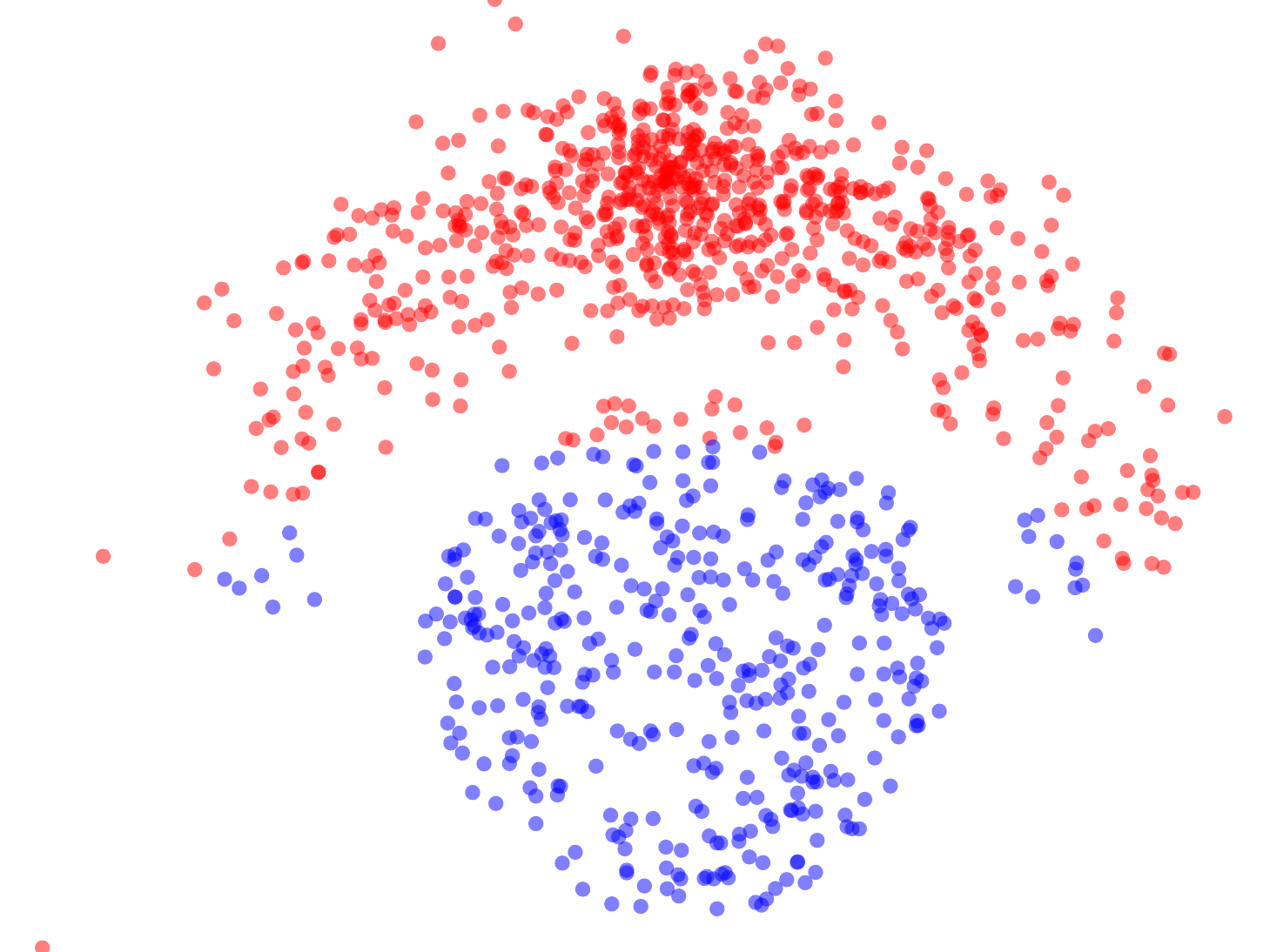} };
        \node at (0.55in,-0.45in) {\tiny $\begin{aligned}\sigma&=2.48 \\[-0.5ex] \NC/&=0.207\end{aligned}$};
        \end{tikzpicture}\\[-0.5ex]
         (a) \NC/ for smaller bandwidth
        & (b)  \NC/ for larger bandwidth \\[-0.5ex]
        (bias to ``sparsest'' subsets) & (loss of non-linear separation)
         \end{tabular}
        \caption{\label{fig:nc fails} Normalized Cut with kernel~\eqref{eq:gauss_k} on the same data as in Figure~\ref{fig:synexp}(c,d).  
        For small bandwidths \NC/ shows bias to small isolated subsets (a). As bandwidth increases, the first non-trivial solution 
        overcoming this bias (b) requires bandwidth large enough so that problems with non-linear separation become visible. 
        Indeed, for larger bandwidths the node degrees become more uniform $d_p\approx \const$
        reducing \NC/ to average association, which is known to degenerate into basic K-means (see Section \ref{sec:extreme}). 
        Thus, any further increase of $\sigma$ leads to solutions even worse than (b). In this simple example
        no fixed $\sigma$ leads \NC/ to a good solution as in Figure~\ref{fig:synexp}(d). That good solution uses adaptive kernel from 
        Section \ref{sec:density_equalization} making specific clustering criterion (\AA/, \NC/, or \AC/) irrelevant, see \eqref{eq:clustering equivalence}.}
\end{figure}
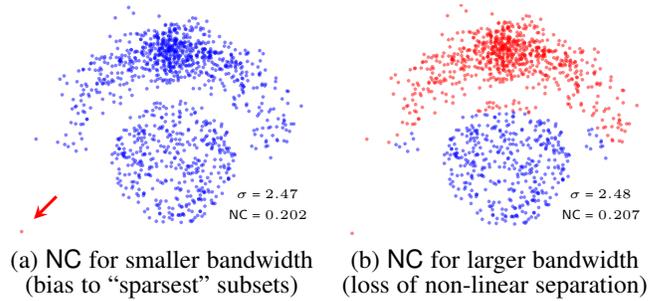

\subsection{Normalization as density inversion} \label{sec:density inversion}

The bias to sparse clusters  in \NC/ with small bandwidths (Figure~\ref{fig:nc fails}a) seems the opposite of mode isolation in \AA/ (Figure~\ref{fig:synexp}c).
Here we show that this observation is not a coincidence since \NC/ can be reduced to \AA/ after some density-inverting data transformation. 
While it is known \cite{bach:nips03,dhillon2004kernel} that \NC/ is equivalent to {\em weighted} kernel K-means 
(\ie {\em weighted} \AA/) with some modified affinity, this section relates such kernel modification 
to an implicit density-inverting embedding where {\em mode isolation} (\thebias/) corresponds to {\em sparse clusters} in the original data.

First, consider standard weighted \AA/ objective for any given affinity/kernel matrix $\hat{A}_{pq}=k(f_p,f_q)$ as in \eqref{eq:wAA} 
$$-\sum_k \frac{\sum_{pq\in S_k} w_p w_q \hat{A}_{pq}}{\sum_{p\in S_k} w_p} .$$
Clearly, weights based on node degrees $w=d$ and ``normalized'' affinities $\hat{A}_{pq}=\frac{A_{pq}}{d_p d_q}$ turn this into \NC/ objective \eqref{eq:NC}.
Thus, average association \eqref{eq:AA} becomes \NC/ \eqref{eq:NC} after two modifications:
\begin{itemize}
\item replacing $A_{pq}$ by normalized affinities $\hat{A}_{pq}=\frac{A_{pq}}{d_p d_q}$ and
\item introducing point weights $w_p=d_p$.
\end{itemize}
Both of these modifications of \AA/ can be presented as implicit data transformations modifying denisty.
In particular, we show that the first one ``inverses'' density turning sparser regions into denser ones, 
see Figure~\ref{fig:density_plot_Amod}(a). The second data modification is generally discussed as a density transform in \eqref{eq:wDensity}. 
We show that node degree weights $w_p=d_p$ do not remove the ``density inversion''. 
}

\pgfkeys{/pgf/fpu}
\pgfmathparse{16383+1}
\edef\tmp{\pgfmathresult}
\pgfkeys{/pgf/fpu=false}

\begin{figure}
        \centering
        \begin{tabular}{cc}
        \begin{tikzpicture}{1.5in}
        \node (A) at (3,7.1) [right] {$x$};
        \draw[->] (-0.1,7.1) -- (A);	  
	  \draw[->] (0,7.05) -- (0,9);
	  \draw[thick,domain=1.15:14,smooth,variable=\x,blue] plot ({\x/5},{exp(16+ln(\x)+6.9-10*ln(7.9+ln(\x)))});
	  \draw (2,7.05) -- (2,7.15);
	  \node at (2,7.05) [below] {$10^4$};
	  \node at (0,7.05) [below] {$0$};
        \node at (2,8.4) {$\tau(x) = \frac{x}{(1+\log x)^{10}}$};
        \end{tikzpicture}  &
        \begin{tikzpicture}{1.5in}
        \node (A) at (4,6.6) [right] {$x$};
        \draw[->] (0.4,6.6) -- (A);	  
	  \draw[->] (0.5,6.5) -- (0.5,8.5);
	  \draw[thick,domain=20:100,smooth,variable=\x,blue] plot ({\x/25},{exp(10+2*ln(\x)-10*ln(1+ln(\x)))});
	  \draw (3,6.55) -- (3,6.65);
	  \node at (3,6.55) [below] {$75$};
	  \draw (1,6.55) -- (1,6.65);
	  \node at (1,6.55) [below] {$25$};
        \node at (3,8) {$\tau(x) = \frac{x^2}{(1+\log x)^{10}}$};
        \end{tikzpicture}   \\[-0.5ex]
         (a) density transform \eqref{eq:density_transform_Amod}   & (b) density transform \eqref{eq:density_transform_Amod_weights} \\[-0.5ex]
         {\footnotesize (kernel normalization only)} & {\footnotesize (with additional point weighting)}
        \end{tabular}
        \caption{ \label{fig:density_plot_Amod} ``Density inversion'' in sparse regions.
        Using node degree approximation $d_p \propto \rho_p$ \eqref{eq:degree=density}
        we show representative density transformation plots (a) $\bar{\rho}_p=\tau(\rho_p)$  and (b) $\rho'_p=\tau(\rho_p)$ 
        corresponding to \AA/ with kernel modification $\hat{A}_{pq}=\frac{A_{pq}}{d_p d_q}$ \eqref{eq:density_transform_Amod}  and 
        additional point weighting $w_p=d_p$ \eqref{eq:density_transform_Amod_weights} exactly corresponding to \NC/.
        This additional weighting weakens the density inversion in (b) compared to (a), see the $x$-axis scale difference.
        However, it is easy to check that the minima in \eqref{eq:density_transform_Amod} and \eqref{eq:density_transform_Amod_weights} 
        are achieved at some $x^*$ exponentially growing with $\bar{N}$.
        This makes the density inversion significant  for \NC/ since $\bar{N}$ may equal the data size.}
\end{figure}
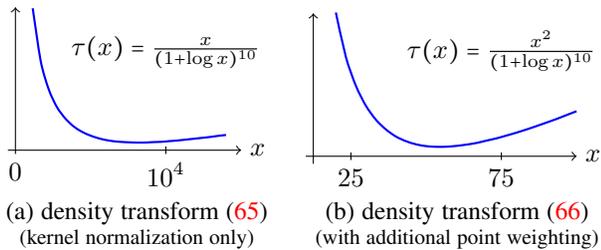

{
For simplicity, assume standard Gaussian kernel \eqref{eq:gauss_k} based on Euclidean distances $d_{pq}=\|f_p-f_q\|$ in $\Real^N$
$$A_{pq}=\exp\frac{-d^2_{pq}}{2\sigma^2} .$$ 
To convert \AA/ into \NC/ we first need an affinity ``normalization''
\begin{equation} \label{eq:modA}
\hat{A}_{pq} =\frac{A_{pq}}{d_p d_q}=\exp\frac{-d^2_{pq}-2\sigma^2\log (d_p d_q)}{2\sigma^2}=\exp\frac{-\hat{d}^2_{pq}}{2\sigma^2}
\end{equation}
equivalently formulated as a modification of distances
\begin{equation} \label{eq:mod_dist}
\hat{d}^2_{pq} \;\;:=\;\; d^2_{pq}+2\sigma^2\log (d_p d_q).
\end{equation}
Using a general approach in the proof of Theorem \ref{th:equiv_clustering}, 
there exists some Euclidean embedding $\bar{f}_p\in\Real^{\bar{N}}$ and constant $h\geq0$ such that
\begin{equation} \label{eq:dist2}
\bar{d}^2_{pq} \;\;:=\;\; \|\bar{f}_p-\bar{f}_q\|^2 \;\; =\;\;  \hat{d}^2_{pq} +h^2[p\neq q].
\end{equation}
Thus, modified affinities $\hat{A}_{pq}$ in \eqref{eq:modA} correspond to the Gaussian kernel  for the new embedding $\{\bar{f}_p\}$ in  $\Real^{\bar{N}}$
$$\hat{A}_{pq}\;\;\propto\;\;\exp\frac{-\bar{d}^2_{pq}}{2\sigma^2}  \;\; \equiv \;\;\exp\frac{-\|\bar{f}_p-\bar{f}_q\|^2}{2\sigma^2} .$$

Assuming $d_q\approx d_p$ for features $f_q$ near $f_p$, equations \eqref{eq:mod_dist} and \eqref{eq:dist2}
imply the following relation for such neighbors of $f_p$ 
$$\bar{d}^2_{pq} \;\;\approx\;\; d^2_{pq} + h^2 + 4\sigma^2\log (d_p). $$
Then, similarly to the arguments in \eqref{eq:density derivation},
a small ball of radius $\varepsilon$ centered at $f_p$ in $\Real^N$ and a ball of radius $\sqrt{\varepsilon^2 + h^2 + 4\sigma^2\log (d_p)}$ 
at $\bar{f}_p$ in $\Real^{\bar{N}}$ contain the same number of points.
Thus, similarly to \eqref{eq:density_transform} we get a relation between densities at points $f_p$ and $\bar{f}_p$
\begin{equation} \label{eq:density_transform_Amod}
\bar{\rho}_p  \;\;\approx\;\; \frac{\rho_p\; \varepsilon^N }{ (\varepsilon^2 + h^2 + 4\sigma^2\log (d_p))^{\bar{N}/2}}.
\end{equation}
This implicit density transformation is shown in Figure~\ref{fig:density_plot_Amod}(a). Sub-linearity in dense regions
addresses mode isolation (\thebias/). However, sparser regions become relatively dense and kernel-modified \AA/ may split them.
Indeed, the result in Figure~\ref{fig:nc fails}(a) can be obtained by \AA/ with normalized affinity $\frac{A_{pq}}{d_p d_q}$.

\begin{figure}
        \centering
        \begin{tabular}{ccc}
	   \includegraphics[height=0.8in,trim=12cm 0 0 0,clip]{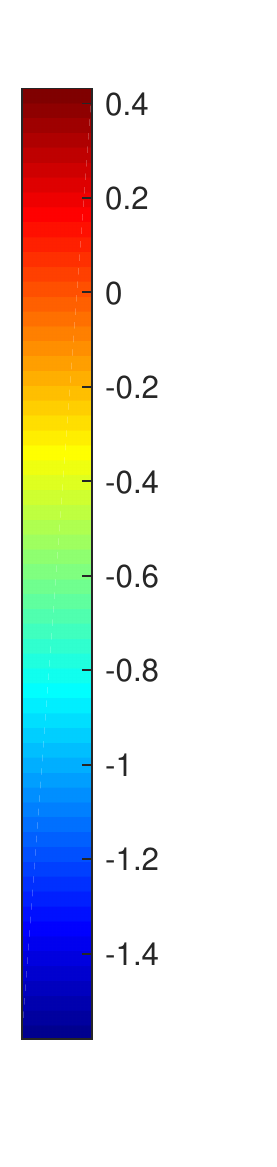} \hspace{-2ex} &
	  \includegraphics[width=1.5in,trim=1cm 1cm 0 2cm,clip]{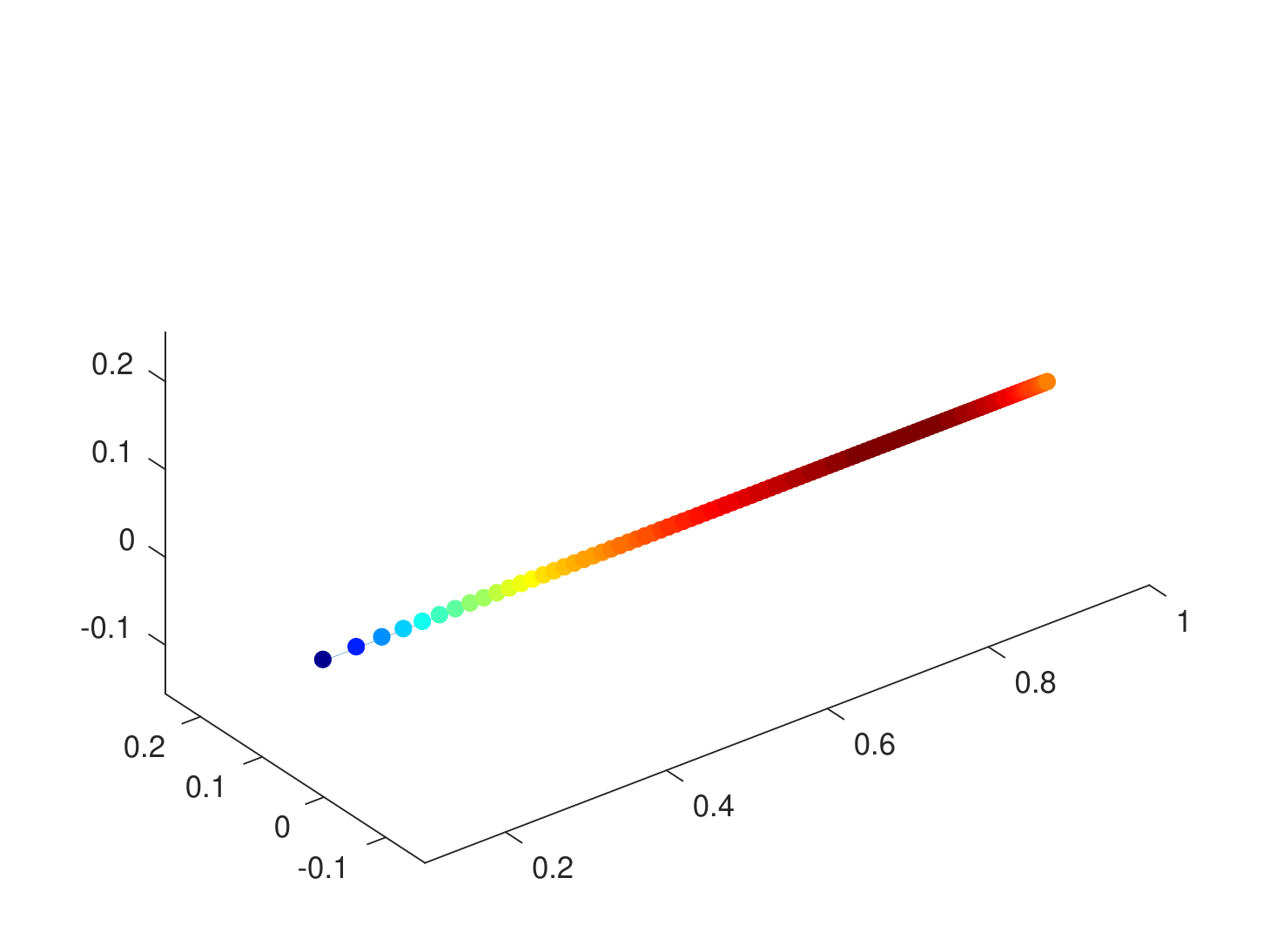} &
	  \includegraphics[width=1.5in,trim=1cm 1cm 0 2cm,clip]{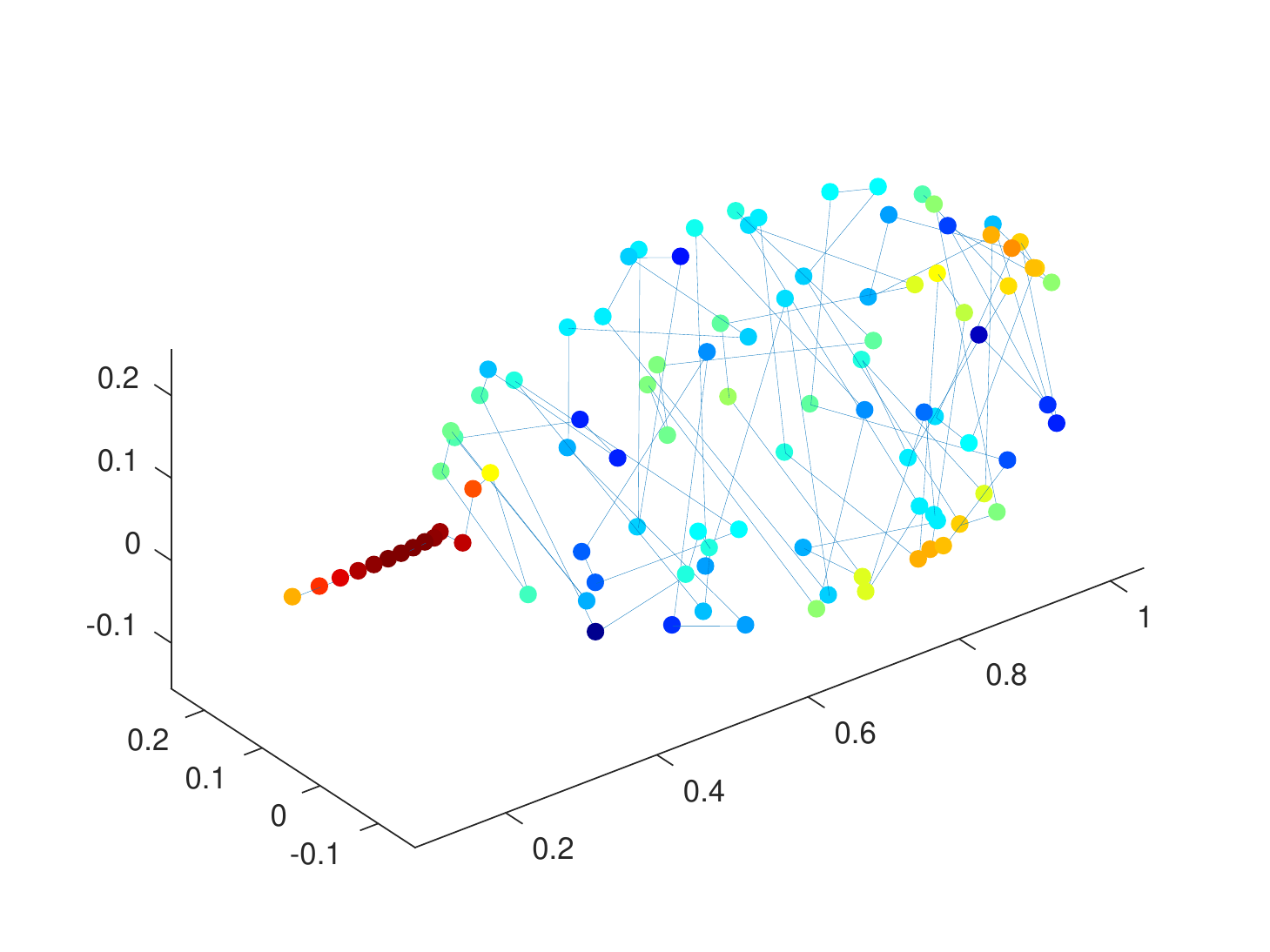}  \\[-0.5ex]
         & (a) original data $\{f_p\}\subset\Real^1$    & (b) embedding $\{\bar{f}_p\}\subset\Real^{\bar{N}}$ 
        \end{tabular}
        \caption{\label{fig:inverse density embedding} Illustration of ``density inversion'' for 1D data.
        The original data points (a) are getting progressively denser along the line. The points are color-coded according to the log of their density.
        Plot (b) shows 3D approximation $\{y_p\}\subset \Real^3$ of high-dimensional Euclidean embedding $\{\bar{f}_p\}\subset\Real^{\bar{N}}$
        minimizing metric errors $\sum_{pq} ( \hat{d}^2_{pq} - \|y_p-y_q\|^2 )^2$ where $\hat{d}_{pq}$ are distances \eqref{eq:mod_dist}.}
\end{figure}
 
The second required modification of \AA/ introduces point weights $w_p=d_p$. 
It has an obvious equivalent formulation via data points replication discussed in Section \ref{sec:adaptive weighting}, see Figure~\ref{fig:density equalization}(a). 
Following \eqref{eq:wDensity}, we obtain its implicit density modification effect $\rho'_p = d_p \bar{\rho}_p$.
Combining this with density transformation  \eqref{eq:density_transform_Amod} 
implied by affinity normalization $\frac{A_{pq}}{d_p d_q}$, we obtain the following density transformation effect corresponding to \NC/, 
see Figure~\ref{fig:density_plot_Amod}(b),
\begin{equation} \label{eq:density_transform_Amod_weights}
\rho'_p  \;\;\approx\;\; \frac{d_p \; \rho_p\; \varepsilon^N }{ (\varepsilon^2 + h^2 + 4\sigma^2\log (d_p))^{\bar{N}/2}}.
\end{equation}
The density inversion in sparse regions relates \NC/'s result in Figure~\ref{fig:nc fails}(a) to 
\thebias/ for embedding $\{\bar{f}_p\}$ in $\Real^{\bar{N}}$.

Figure \ref{fig:density_plot_Amod} shows representative plots for density transformations  \eqref{eq:density_transform_Amod},
\eqref{eq:density_transform_Amod_weights} using the following node degree approximation based on Parzen approach \eqref{eq:Parzen_estimate}
for Gaussian affinity (kernel) $A$
\begin{equation} \label{eq:degree=density}
d_p = \sum_q A_{pq} \propto \rho_p.
\end{equation} 

{
\setlength{\columnsep}{0.5ex}%
\setlength{\intextsep}{-1ex}%

Empirical relation between $d_p$ and $\rho_p$ is illistrated below: some
\begin{wrapfigure}{r}{42mm}
        \begin{tikzpicture}{35mm}
	  \node at (0,1.8) [right] {$d_p$ - node degree};
        \draw[->] (-0.1,0) -- (3.5,0);	  
	  \draw[->] (0,-0.1) -- (0,1.8);
	  \draw[dashed,domain=0:3.5,smooth,variable=\x,pink] plot ({\x},{\x/2});
	  \draw[dashed,domain=-0.05:3.5,smooth,variable=\x,pink] plot ({\x},{1.3});
	  \node at (0,1.3) [red,left] {$N$};
	  \node at (3,1.4) [red,rotate=30,above] {$d_p \sim \rho_p$};
	  \draw[thick,domain=0:3.5,smooth,variable=\x,black] plot ({\x},{1.4-exp(-\x/2)*1.1});
	  \node at (3,1.2) [below] {$d_p(\rho_p)$};
	  \node at (0,0.3) [left] {$1$};
	  \draw (-0.05,0.3) -- (0.05,0.3);
	  \node at (-0.1,0) [below] {$0$};
        \node at (3.5,0.2) [left] {$\rho_p$ - density};
        \end{tikzpicture}  
\end{wrapfigure}
overestimation occurs for sparcer regions and underestimation happens for denser regions.
The node degree for Gaussian kernels has to be at least $1$ (for an isolated node) and at most $N$ (for a dense graph).

}

}

\section{Discussion {\small (kernel clustering equivalence)}} \label{sec:discussion}

Density equalization with adaptive weights in Section \ref{sec:adaptive weighting} or adaptive kernels in Section \ref{sec:Bandwidth selection} 
are useful for either \AA/ or \NC/ due to their density biases (mode isolation or sparse subset). 
Interestingly, kernel clustering criteria discussed in \cite{Shi2000} such as normalized cut (\NC/), {\em average cut} (\AC/), 
average association (\AA/) or kernel K-means are practically equivalent for such adaptive methods.
This can be seen both empirically (Table~\ref{tb:errorrates}) and conceptually. 
Note, weights $w_p\propto 1/\rho_p$ in Section \ref{sec:adaptive weighting}
produce modified data with near constant node degrees $d'_p \propto \rho'_p \propto 1$, 
see \eqref{eq:degree=density} and \eqref{eq:wDensity}.
Alternatively, KNN kernel (Example~\ref{ex:KNN kernel}) with density equalizing bandwidth \eqref{eq:knn bandwidth} also produce 
nearly constant node degrees $d_p \approx K$ where $K$ is the neighborhood size. Therefore, both cases give
\begin{equation}  \label{eq:clustering equivalence}
 -\frac{\sum_{pq \in S^k} A_{pq}}{\sum_{p \in S^k}d_p}  \;\; \propto \;\;   
 -\frac{\sum_{pq \in S^k} A_{pq}}{K\,|S^k|}  \;\;\approxutc\;\;  \frac{\sum_{p \in S^k, q \in \bar{S}^k} A_{pq}}{ K\,|S^k|} ,  
\end{equation} 
which correspond to \NC/ \eqref{eq:NC}, \AA/  \eqref{eq:AA}, and \AC/ criteria. As discussed in \cite{Shi2000}, the last objective
also has very close relations with standard partitioning concepts in spectral graph theory: {\em isoperimetric} or {\em Cheeger number},
{\em Cheeger set}, {\em ratio cut}. 

This equivalence argument applies to the corresponding clustering objectives and is independent of specific optimization algorithms developed for them. 
Interestingly, the relation between \eqref{eq:AA} and basic K-means objective \eqref{eq:kmixed} suggests that standard Lloyd's algorithm 
can be used as a basic iterative approach for approximate optimization of all clustering criteria in \eqref{eq:clustering equivalence}.
In practice, however, kernel K-means algorithm corresponding to the exact high-dimensional embedding $\{\phi_p\}$ in \eqref{eq:kmixed}
is more sensitive to local minima compared to iterative K-means over approximate lower-dimensional embeddings based on PCA 
\cite[Section 3.1]{KC:arXiv16}\footnote{K-means is also commonly used as a discretization heuristic for {\em spectral relaxation} \cite{Shi2000}
where a similar eigen analysis is motivated by spectral graph theory \cite{Cheeger:70, Hoffman:73, Fiedler:75} 
defferently from PCA dimensionalty reduction in \cite{KC:arXiv16}.}.

{\NEW
\section{Conclusions}

This paper identifies and proves density biases, \ie isolation of modes or sparsest subsets, 
in many well-known kernel clustering criteria such as kernel K-means (average association), ratio cut, normalized cut, dominant sets. 
In particular, we show conditions when such biases happen.
Moreover,  we propose density equalization as a general principle for resolving such biases. We suggest
two types of density equalization techniques using adaptive weights or adaptive kernels. 
We also show that density equalization unifies many popular kernel clustering objectives by making them equivalent.
}

\section*{Acknowledgements}
The authors would like to thank Professor Kaleem Siddiqi (McGill University) for 
suggesting a potential link between \thebias/ and the {\em dominant sets}.
This work was generously supported by the Discovery and RTI programs of the 
National Science and Engineering Research Council of Canada (NSERC).

{

}

\begin{IEEEbiography}[{\includegraphics[width=1in,height=1.25in,clip,keepaspectratio]{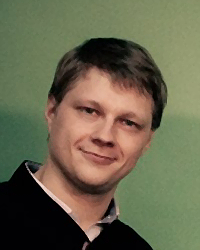}}]{Dmitrii Marin} received Diploma of Specialist from the Ufa State Aviational Technical University in 2011, and M.Sc. degree in Applied Mathematics and Information Science from the National Research University Higher School of Economics, Moscow, and graduated from the Yandex School of Data Analysis, Moscow, in 2013. In 2010 obtained a certificate of achievement at ACM ICPC World Finals, Harbin. 
He is a PhD candidate at the Department of Computer Science, University of Western Ontario under supervision of Yuri Boykov. His research is focused on designing general unsupervised and semi-supervised methods for accurate image segmentation and object delineation.
\end{IEEEbiography}

\begin{IEEEbiography}[{\includegraphics[width=1in,height=1.25in,clip,keepaspectratio]{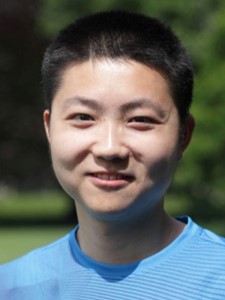}}]{Meng Tang} is a PhD candidate in computer science at the University of Western Ontario, Canada, supervised by Prof. Yuri Boykov. He obtained MSc in computer science in 2014 from the same institution for his thesis titled "Color Separation for Image Segmentation". Previously in 2012 he received B.E. in Automation from the Huazhong University of Science and Technology, China. He is interested in image segmentation and semi-supervised data clustering. He is also obsessed and has  experiences on discrete optimization problems for computer vision and machine learning.
\end{IEEEbiography}

\begin{IEEEbiography}[{\includegraphics[width=1in,height=1.25in,clip,keepaspectratio]{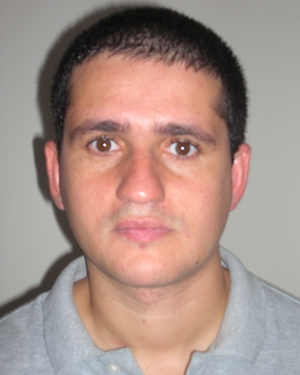}}]{Ismail Ben Ayed} received the PhD  degree
 (with the highest honor) in computer vision from the Institut National de la Recherche Scientifique (INRS-EMT), Montreal, QC, in 2007. 
He is currently Associate Professor at the Ecole de Technologie Superieure (ETS), University of Quebec, where he holds a research chair on Artificial Intelligence in Medical Imaging. Before joining the ETS, he worked for 8 years as a research scientist at GE Healthcare, London, ON, conducting research in medical image analysis. He also holds an adjunct professor appointment at the University of Western Ontario (since 2012). Ismail's research interests are in computer vision, optimization, machine learning and their potential applications in medical image analysis.
\end{IEEEbiography}

\begin{IEEEbiography}[{\includegraphics[width=1in,height=1.25in,clip,keepaspectratio]{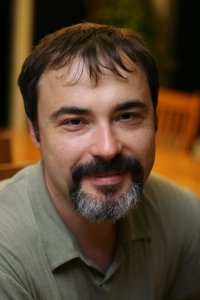}}]{Yuri Boykov}
received "Diploma of Higher Education" with honors at Moscow Institute of Physics and Technology (department of Radio Engineering and Cybernetics) in 1992 and completed his Ph.D. at the department of Operations Research at Cornell University in 1996.
He is currently a full professor at the department of Computer Science at the University of Western Ontario. His research is concentrated in the area of computer vision and biomedical image analysis. In particular, he is interested in problems of early vision, image segmentation, restoration, registration, stereo, motion, model fitting, feature-based object recognition, photo-video editing and others. He is a recipient of the Helmholtz Prize (Test of Time) awarded at International Conference on Computer Vision (ICCV), 2011 and Florence Bucke Science Award, Faculty of Science, The University of Western Ontario, 2008.
\end{IEEEbiography}

\end{document}